\documentclass[11pt]{article} 
\usepackage{makecell}
\usepackage{fullpage}
\usepackage{graphics} 
\usepackage{epsfig}
\usepackage{amsmath} 
\usepackage{amssymb} 
\usepackage{amsthm}
\usepackage{tabularx,booktabs}
\usepackage{mathrsfs}
\usepackage{selectp}
\usepackage[normalem]{ulem}
\usepackage{authblk}
\usepackage{subcaption}
\usepackage{natbib}
\usepackage{Shorthands}
\usepackage{algorithm}
\usepackage{algpseudocode}
\usepackage{wrapfig}
\usepackage{scalerel,stackengine}
\stackMath
\newcommand\reallywidehat[1]{%
\savestack{\tmpbox}{\stretchto{%
  \scaleto{%
    \scalerel*[\widthof{\ensuremath{#1}}]{\kern-.6pt\bigwedge\kern-.6pt}%
    {\rule[-\textheight/2]{1ex}{\textheight}}
  }{\textheight}%
}{0.5ex}}%
\stackon[1pt]{#1}{\tmpbox}%
}

\usepackage{hyperref}
\hypersetup{       
    unicode=false,          
    pdftoolbar=true,       
    pdfmenubar=true,      
    pdffitwindow=false,     
    pdfstartview={FitH},   
    pdfnewwindow=true,      
    colorlinks=true,      
    linkcolor=blue,          
    citecolor=blue,       
    filecolor=magenta,      
    urlcolor=blue,           
    breaklinks=true
}

\usepackage[toc,page,header]{appendix}
\usepackage{minitoc}

\usepackage{amsmath}
\usepackage{algorithm}
\usepackage{algpseudocode}
\usepackage{mathtools}
\usepackage[dvipsnames]{xcolor}
\usepackage[dvipsnames,table]{xcolor}
\usepackage{pifont}
\usepackage{graphicx}
\usepackage{subcaption}
\makeatletter
\renewcommand\@makefntext[1]{%
  \parindent 1em
  \noindent
  \hb@xt@1.8em{\hss\@makefnmark\ }#1%
}
\makeatother
\usepackage{tikz}
\usetikzlibrary{arrows.meta,positioning,fit,calc,shapes.geometric}
\usepackage{graphicx} 

\usepackage[most]{tcolorbox}

\newtcolorbox{goalbox}{
  colback=gray!5!white,
  colframe=gray!60!black,
  boxrule=0.5pt,
  arc=2pt,
  left=5pt,right=5pt,top=3pt,bottom=3pt
}

\usepackage[toc,page,header]{appendix}
\usepackage{minitoc}

\title{\LARGE \bf
Adversarially Robust Multitask Adaptive Control
}

\author{Kasra Fallah*}
\author{Leonardo F. Toso*}
\author{James Anderson}
\affil{Department of Electrical Engineering, Columbia University}
\date{November 2025}
\begin{document}

\doparttoc 
\faketableofcontents 

\maketitle
\allowdisplaybreaks
\begin{abstract}
We study \emph{adversarially robust} multitask adaptive linear quadratic control; a setting where multiple systems collaboratively learn control policies under model uncertainty and adversarial corruption. We propose a clustered multitask approach that integrates clustering and system identification with \emph{resilient aggregation} to mitigate corrupted model updates. Our analysis characterizes how clustering accuracy, intra-cluster heterogeneity, and adversarial behavior affect the expected regret of certainty-equivalent (CE) control across LQR tasks. We establish non-asymptotic bounds demonstrating that the regret decreases inversely with the number of honest systems per cluster and that this reduction is preserved under a bounded fraction of adversarial systems within each cluster.
\end{abstract}

\footnote{\hspace{-0.6cm}*K.~Fallah and L.~F.~Toso share first-authorship. \\ Correspondence to: \texttt{\{kasra.fallah, leonardo.toso\}@columbia.edu}.}

\section{Introduction}

Adaptive control seeks to design controllers that adapt to uncertain or unknown system dynamics. Rooted in early work on self-tuning regulators for flight and aerospace applications \citep{aastrom1973self,aastrom1983theory}, it remains central to modern control. Among its formulations, the linear quadratic regulator (LQR) serves as a canonical benchmark due to its tractability and theoretical appeal. Extensive research over the last five or so years has established \emph{non-asymptotic} performance guarantees for adaptive LQR through regret analysis \citep{abbasi2011regret, dean2018regret, cohen2019learning, simchowitz2020naive, hazan2020nonstochastic, ziemann2022regret}, proving that in the single-system setting the optimal expected regret scales as $\mathcal{O}(\sqrt{dT})$, with $d = d_u^2 d_x$, where $T$ is the time horizon and $(d_x, d_u)$ denote the state and input  dimensions \citep{simchowitz2020naive}. This lower bound reveals a fundamental limitation: certainty-equivalent (CE) control is inherently data-inefficient in high dimensions, as accurate model estimation demands extensive data collection.

To circumvent this limitation, recent work has investigated \emph{multitask} system identification, where multiple systems collaboratively estimate their dynamics. When the participating systems are ``similar'', collaboration reduces the sample complexity required for accurate model estimation, with gains that scale proportionally with the number of participating systems \citep{xin2022identifying, wang2023fedsysid, toso2023learning, kecceci2025fedalign}. When the systems are homogeneous or share some model structure (e.g., a model basis), joint estimation yields improvements in sample complexity \citep{zhang2024sampleefficient}. On the other hand, when systems are only approximately similar, an additive \emph{heterogeneity} bias emerges \citep{wang2023fedsysid}. As the primary source of regret in adaptive control stems from identification, multitask identification offers a natural path to surpass the $\mathcal{O}(\sqrt{T})$ scaling limit from the single-task setting.

\begin{figure}
  \centering
  \includegraphics[width=1\textwidth]{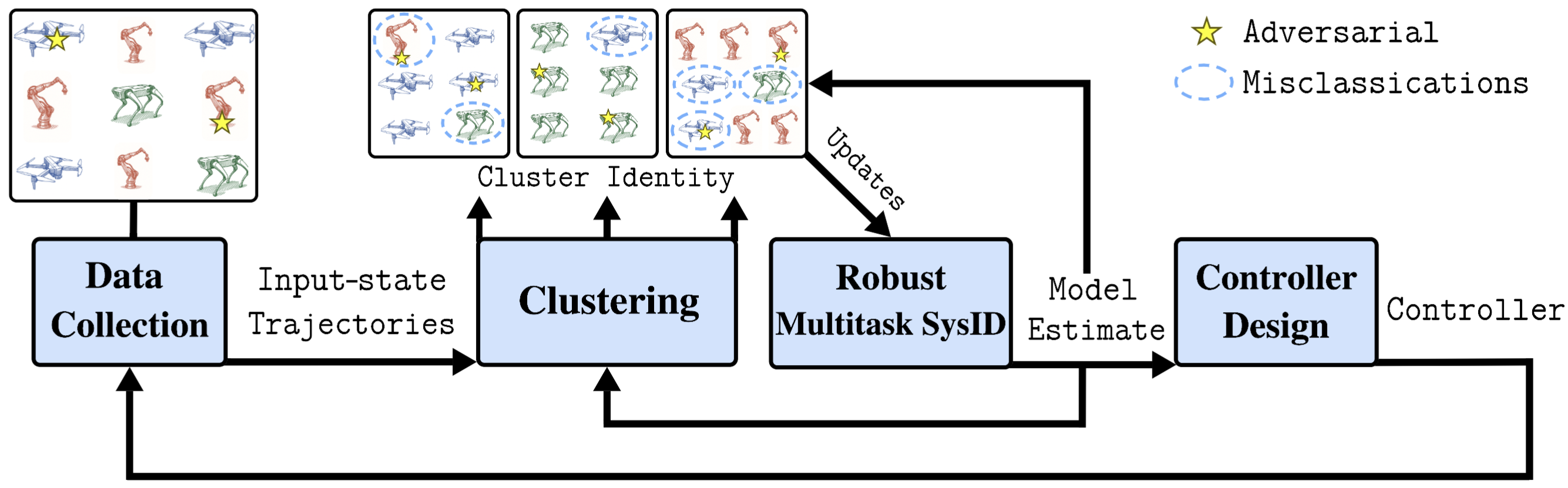}
  \caption{Workflow for adversarially robust multitask adaptive control$^\dagger$.}
  \label{fig:multitask_control_pipeline} 
\end{figure}

\footnotetext[0]{\hspace{-0.5cm}$^\dagger$~Illustrations of drones, quadruped robots, and robotic arms were created with assistance from ChatGPT (OpenAI).}

\renewcommand{\thefootnote}{\arabic{footnote}}
Recent work has extended the idea of multitask system identification \citep{wang2023fedsysid, zhang2024sampleefficient} to the control synthesis setting \citep{wang2023model, toso2024meta, wang2023fleet, lee2024regret}. \cite{lee2024regret} prove that learning a shared model basis across systems allows the expected regret to scale as $\mathcal{O}(\sqrt{T / \#\text{systems}})$, reducing the regret by the number of systems. However, this result relies on \emph{structural} homogeneity, namely, the existence of a shared model representation across all participating systems.

In practice,  multitask control systems, such as fleets of drones, autonomous vehicles, and distributed robotic platforms \citep{wang2023fleet} often exhibit diverse dynamics without a shared representation. In such settings, malfunctioning or compromised systems may transmit corrupted updates, undermining collaboration. This motivates the need for adversarially robust multitask adaptive control, where learning remains effective despite heterogeneous and adversarial systems.

This work studies clustered multitask adaptive control in the presence of \emph{heterogeneous} and \emph{adversarial} systems (see Figure~\ref{fig:multitask_control_pipeline}), when a common representation across systems \emph{may not exist}. We analyze how malicious systems can bias the collaborative learning step and design a robust approach to mitigate such behavior. Our analysis characterizes the interplay among adversarial behavior, intra-cluster heterogeneity, and clustering accuracy on the regret. We prove that the regret scales favorably with the number of systems per cluster. This benefit dominates even under a small fraction of adversarial systems per cluster. This provides the first non-asymptotic regret bounds for clustered multitask adaptive control under heterogeneous and adversarial systems (see Table \ref{tb:comp}). We now present an informal statement of our main result.

\begin{theorem}[Informal]\label{thm:informal} \hspace{-0.15cm}Consider the multitask adaptive linear quadratic control pipeline illustrated in Figure~\ref{fig:multitask_control_pipeline}. For an appropriate choice of exploration and a sufficiently large amount of data per system, let $T$ denote the time horizon, $m_j$ the number of honest systems in cluster $\mathcal{C}_j$, and $\lambda$ the resilient aggregation coefficient. Then, the expected regret of any system in $\mathcal{C}_j$ satisfies: 

\begin{align*}
\begin{array}{c}\texttt{Regret} \end{array} \lesssim
\hspace{-0.5cm} \underbrace{\sqrt{\frac{dT}{\textcolor{Green}{m_{j}}}}}_{\textcolor{Green}{\text{benefit of multitask}}} \hspace{-0.1cm} + \hspace{-0.1cm}
\underbrace{\begin{array}{c}  
\textcolor{Maroon}{\texttt{cluster error}}\end{array}
\hspace{-0.1cm} \times  T}_{\textcolor{Maroon}{\text{clustering effect}}} +  \underbrace{\textcolor{red}{\lambda} \sqrt{d T}}_{\textcolor{red}{\text{adversarial effect}}} \hspace{-0.1cm} + \underbrace{ \textcolor{red}{\texttt{heterogeneity}}\times T}_{\textcolor{red}{\text{heterogeneity effect}}},\\[-1cm]
\end{align*}
\noindent where $\begin{array}{c} \textcolor{Maroon}{\texttt{cluster error}} \end{array} \lesssim \exp(\textcolor{Green}{\texttt{-\#~data samples per system}}).$
\end{theorem}

Theorem~\ref{thm:informal} highlights the benefits of multitask adaptive control and quantifies the effects of clustering errors, adversarial systems, and heterogeneity. The first term captures the benefit of \emph{collaboration}, reducing regret proportionally to the number of honest systems $m_j$. The second term reflects the clustering misclassification rate that decays exponentially with data size. The third term represents the effect of adversarial updates, governed by the resilient aggregation parameter $\lambda$, which scales with the ratio of adversarial to honest systems \citep{farhadkhani2022byzantine}, when this ratio is small, the collaboration benefit is preserved. The final term captures the \emph{heterogeneity} effect, which is negligible under small intra-cluster heterogeneity. We characterize the first two terms in Theorem~\ref{theorem: regret homogeneous}, the adversarial effect in Theorem~\ref{theorem: regret adversarial}, and the heterogeneity effect in Corollary~\ref{corollary: regret heterogeneous}.

\subsection{Contributions}

\noindent $\bullet$ \textbf{Heterogeneous Systems.} We study multitask adaptive LQR under heterogeneous dynamics where  global representation across system models are not assumed. We propose a clustered system identification approach that groups similar systems and performs multitask identification within each cluster. We derive non-asymptotic estimation error bounds demonstrating that sample complexity improves proportionally to the number of systems per cluster, up to an exponentially small cluster misclassification rate (Lemma \ref{lemma: sysID homogeneous} and Proposition \ref{prop: sysID heterogeneous}).\vspace{0.05cm}  

\noindent $\bullet$ \textbf{Adversarial Systems.} Our approach handles adversarial systems that may contribute with corrupted model updates. Robustness is ensured with a \emph{resilient} aggregation step that may employ any resilient aggregation rule \citep{farhadkhani2022byzantine}. Under a bounded fraction of adversarial systems, our approach preserves the  asymptotic gains as in the fully honest setting (Lemma \ref{lemma: sysID adversarial}).  \vspace{0.05cm}  

\noindent $\bullet$ \textbf{Regret Bounds.} We provide the first non-asymptotic regret bounds for adversarially robust multitask adaptive LQR. The analysis quantifies the impact of clustering misclassifications, intra-cluster heterogeneity, and adversarial behavior, demonstrating that with sufficient data, small intra-cluster heterogeneity, and bounded fraction of adversarial systems these effects are negligible, preserving the regret reduction by the number of honest systems (Theorems \ref{theorem: regret homogeneous} and \ref{theorem: regret adversarial}, and Corollary \ref{corollary: regret heterogeneous}).

\subsection{Related Work}\label{sec: related work}

Recent work in multitask system identification have explored collaborative learning to improve sample complexity. \citet{wang2023fedsysid} first proved that the sample complexity of learning linear dynamical systems decreases with the number of collaborating systems, up to a heterogeneity bias. This bias is replaced by a misclassification term in clustered system identification~\citep{toso2023learning, kecceci2025redefining}, where systems are assumed identical within clusters and the misclassification rate decays exponentially with the amount of local data. \citet{zhang2024sampleefficient} instead assumes a shared latent representation across system models, implicitly requiring strong structural similarity. In contrast, our work considers multitask clustered system identification under intra-cluster heterogeneity, allowing for settings where no global representation exists.

Within adaptive control, the most relevant work is \citet{lee2024regret}, which learns a shared basis across systems to accelerate policy adaptation. Our framework removes this structural assumption by learning cluster-specific models that capture local similarities while remaining robust to adversarial systems. Adversarially resilient learning has been well studied in distributed and federated settings~\citep{blanchard2017machine, chen2018draco, farhadkhani2022byzantine, dong2023byzantine}, where resilient aggregation ensures robustness against malicious agents. We extend these ideas to the more challenging domain of system identification and control, providing, to the best of our knowledge, the first analysis of adversarially robust multitask adaptive control.

\begin{table}[ht]
    \caption{\small Expected regret for single-and multitask online learning. In the regret bound reported by~\cite{lee2023nonasymptotic}, $d_{w}$ denotes the dimension of the task-specific weight vector, as the model is decomposed into a task weight and a shared representation. Moreover, $\delta_{\texttt{dist}}$ quantifies the discrepancy between the ground-truth representation of the system model and a given pre-trained representation. The setting considered in~\cite{lee2024regret} assumes that the $M$ tasks are structurally similar so that a common representation exists. 
} 
  \label{tb:comp}%
  \centering%
  \resizebox{\textwidth}{!}{%
      \begin{tabular}{lllllll}
        \thead{Work}                                  &\thead{Setting} & \thead{Heterogeneity} & \thead{Adversarially Robust} & \thead{Regret} \\
        \midrule                                                                                         
        
        \cite{simchowitz2020naive} & Single task & \hspace{1cm}\xmark & \hspace{1cm}\xmark & $\mathcal{O}(\sqrt{d T})$ \\
        
        \cite{lee2023nonasymptotic} & Single task   & \hspace{1cm}\xmark      & \hspace{1cm}\xmark      & $\mathcal{O}(\sqrt{d_{w} d_x T}) + \delta_{\texttt{dist}} T$        \\
        \cite{lee2024regret} & Multitask  & Structurally Homogeneous & \hspace{1cm}\xmark      & $\mathcal{O}\left(\sqrt{\frac{\texttt{poly}(d_x, d_u)T}{M}}\log^2(TM)\right)$\\
         \cite{dong2023byzantine}     & Distributed     & Heterogeneous            & \hspace{1cm}\cmark   & $\mathcal{O}(\sqrt{T} + \epsilon_{\mathrm{het}} T)$                         \\
        \rowcolor{blue!30}\textbf{This work}               & Multitask    & Heterogeneous      & \hspace{1cm}\cmark     & $\mathcal{O}\left(\sqrt{\frac{d T}{m_{j}}} + \lambda\sqrt{d T}+ \epsilon_{\mathrm{het}} T\right)$
      \end{tabular}
}
\end{table}

\subsection{Notation} The norm $\|\cdot\|$ denotes the Euclidean norm for vectors and the spectral norm for matrices, while $\|\cdot\|_F$ is the Frobenius norm. The operators  $\dlyap(A,Q)$ and $\dare(A,B,Q,R)$ denotes the solutions to the discrete Lyapunov and Riccati equations, respectively. We use $\mathcal{O}(\cdot)$ to hide constant factors and $\widetilde{\mathcal{O}}(\cdot)$ to hide logarithmic terms.

\section{Problem Setup}\label{sec: problem setup}

We consider $M$ discrete-time linear time-invariant (LTI) systems, among which an unknown  $f$ systems may behave adversarially. The remaining $m = M - f$ honest systems evolve according to
\begin{align}\label{eq:data generation}
 x^{(i)}_{t+1} = \Theta^{(i)}_\star z^{(i)}_t + w^{(i)}_t, \; \forall t= 0,1,\ldots, \;
z^{(i)}_t=\begin{bmatrix}x^{(i)}_t\\ u^{(i)}_t\end{bmatrix}\in\mathbb{R}^{d^\prime}, \text{ with } d^\prime=d_x+d_u,   
\end{align}
where $\Theta^{(i)}_\star = [A^{(i)}_\star \; B^{(i)}_\star]$ denotes system $i$'s model, and $\{w^{(i)}_t\}_{t}$ has elements that are i.i.d. mean-zero $\sigma^2_w$-sub-Gaussian, for some positive variance proxy $\sigma^2_w$ \citep{vershynin2018high}.\vspace{0.2cm}

\noindent\textbf{Clusters.} Each honest system belongs to one of $N_c$  clusters $\{\mathcal{C}_j\}_{j=1}^{N_c}$, grouping systems with similar dynamics. Systems within the same cluster share model parameters that are ``close" in the Frobenius norm (Assumption \ref{assumption: heterogeneity}). The separation between clusters is characterized by the minimum and maximum distances, $\Delta_{\min} \triangleq \min_{j \neq j'} \|\Theta_j - \Theta_{j'}\|$ and $\Delta_{\max} \triangleq \max_{j \neq j'} \|\Theta_j - \Theta_{j'}\|$, respectively.\vspace{0.2cm}

\noindent\textbf{Control Input.} The control input is composed of a stabilizing and an exploratory component,  $u^{(i)}_t = u^{(i)}_{\mathrm{stab}} + u^{(i)}_{\mathrm{exp}}$, 
where $u^{(i)}_{\mathrm{stab}} = K^{(i)} x^{(i)}_t$ is the stabilization term and  $u^{(i)}_{\mathrm{exp}} = \sigma_u g^{(i)}_t$ provides exploration, with $g^{(i)}_t \overset{\text{i.i.d.}}{\sim} \mathcal{N}(0,I_{d_u})$. For every honest system, the \emph{exploration} term $\sigma_u g^{(i)}_t$ ensures persistent excitation, guaranteeing $\lambda_{\min} \left(\mathbb{E}[z^{(i)}_t z^{(i)\top}_t]\right) \geq \sigma_u^2$, where $z^{(i)}_t \overset{\text{i.i.d.}}{\sim} \mathcal{N}(0,\Sigma^{(i)}_{z_t})$. We refer the reader to \citep[Lemma 1]{wang2023fedsysid} for the definition of the covariance matrix $\Sigma^{(i)}_{z_t}$. \vspace{0.2cm}

\noindent\textbf{Dataset.} We denote the local trajectory data up to time $\tau$ by
\begin{align*} 
X^{(i)}_\tau = [x^{(i)}_1,\ldots,x^{(i)}_\tau], \text{ and }
Z^{(i)}_\tau = \big[[x^{(i)}_0,u^{(i)}_0]^\top, \ldots, [x^{(i)}_{\tau-1},u^{(i)}_{\tau-1}]^\top\big],
\end{align*}
and define system $i$'s dataset $\mathcal{D}^{(i)}_\tau = \{X^{(i)}_\tau, Z^{(i)}_\tau\}$.

\begin{assume}[Intra-cluster heterogeneity]\label{assumption: heterogeneity}
Let $\mathcal{H}_j \subseteq \mathcal{C}_j$  denote the index set of honest systems inside cluster $\mathcal{C}_j$.  
There exists a scalar $\epsilon_{\mathrm{het}} \geq 0$ such that\vspace{-0.1cm}
\begin{align}\label{eq: intracluster heterogeneity}
\max_{j \in [N_c],\, i,\ell \in \mathcal{H}_j}
\big\|[A^{(i)}_\star \; B^{(i)}_\star] - [A^{(\ell)}_\star \; B^{(\ell)}_\star]\big\|_F
\le \epsilon_{\mathrm{het}}.\notag\\[-1cm]
\end{align}
\end{assume}

\subsection{Certainty-Equivalent Linear Quadratic Control}

We fix an honest system $i \in \mathcal{C}_j$ with $u^{(i)}_{\mathrm{exp}} = 0$ (for now). The local control objective is to design a linear feedback controller that minimizes the infinite-horizon quadratic cost

\begin{align*}
J^{(i)}(K) \triangleq \limsup_{T\to \infty} \frac{1}{T}\E_{K} \left[\sum_{t=0}^{T-1} c^{(i)}_t\right], \text{ with } c^{(i)}_{t} \triangleq x_t^{(i)\top}(Q + K^\top R K)x_t^{(i)},\\[-1cm]
\end{align*}
where $Q \succeq 0$ and $R \succ 0$. When the ground-truth dynamics $(A^{(i)}_\star,B^{(i)}_\star)$ are known, the optimal controller
$K_\star^{(i)} = K(A^{(i)}_\star,B^{(i)}_\star) \triangleq -\left(R + B^{(i)\top}_\star P_\star^{(i)} B^{(i)}_\star\right)^{-1} B^{(i)\top}_\star P_\star^{(i)} A^{(i)}_\star$
is obtained by solving the discrete algebraic Riccati equation (DARE), $P_\star^{(i)} = \texttt{DARE}(A^{(i)}_\star, B^{(i)}_\star, Q, R)$.

In practice, the true model parameters $(A^{(i)}_\star, B^{(i)}_\star)$ are typically \emph{unknown} and must be inferred from data.  
A standard approach to designing a near-optimal controller is certainty-equivalent control \citep{mania2019certainty}, which first estimates the system model and then computes the controller as if the estimate were exact. In particular, by using trajectory data $\mathcal{D}^{(i)}_\tau$, each system performs ordinary least-squares (OLS) estimation as

\begin{align}\label{eq:ols_single}
\widehat{\Theta}^{(i)} 
\triangleq [\hat{A}^{(i)} \; \hat{B}^{(i)}]
= \argmin_{\Theta \in \mathbb{R}^{d_x \times d'}} 
\left\|X^{(i)}_\tau - \Theta Z^{(i)}_\tau\right\|_F^2,\notag \\[-1cm]
\end{align}
and designs a controller $\hat{K}^{(i)} = K(\hat{A}^{(i)},\hat{B}^{(i)})$ with the estimated model.

While conceptually simple, applying CE control to a single system, suffers from poor sample efficiency.   This motivates \emph{adaptive} LQR, where learning and control are performed jointly, i.e., data are collected under the current controller, the model is updated, and the control policy is refined iteratively. \cite{simchowitz2020naive} demonstrate that a simple greedy strategy, balancing exploration through $u^{(i)}_{\mathrm{exp}} \neq 0$ and exploitation through the data size $\tau$, achieves the optimal expected regret scaling of $\mathcal{O}(\sqrt{dT})$.  
Nevertheless, as the dominant source of regret arises from system identification, doing so on a single system still requires extensive data. This limitation makes CE control particularly challenging in data-scarce scenarios~\citep{topcu2022learning}.

\subsection{Multitask Adaptive Control} \label{subsec:multitask_LQR}

To overcome the limitations of system identification using data from a single system, we consider a setting in which multiple systems collaborate, with $f = 0$ adversarial systems (for now), to learn a common model that best fits their collective data.  When the true cluster identities are known, multitask system identification can be performed within each cluster $\mathcal{C}_j$ by replacing~\eqref{eq:ols_single} with 
\begin{align*} 
\widehat{\Theta}_j  \triangleq [\hat{A}_j \; \hat{B}_j] = \argmin_{\Theta \in \mathbb{R}^{d_x \times d'}}  \frac{1}{|\mathcal{C}_j|}\sum_{i \in \mathcal{C}_j} \left\|X^{(i)}_\tau - \Theta Z^{(i)}_\tau\right\|_F^2.\\[-1cm]
\end{align*}
The resulting estimate $\widehat{\Theta}_j$ defines a shared model for the cluster, from which a common controller $\widehat{K}^{(i)} = K(\hat{A}_j, \hat{B}_j)$ is designed for all systems $i \in \mathcal{C}_j$.
\renewcommand{\thefootnote}{\arabic{footnote}} As proved in prior work on federated system identification \citep{wang2023fedsysid}, when local gradient updates $G_{\ell}(\widehat{\Theta}_j)$ are shared and aggregated by a trusted server using simple averaging, namely, $ \widehat{\Theta}_j \leftarrow \widehat{\Theta}_j + \eta F \left(\{G_{\ell}(\widehat{\Theta}_j)\}_{\ell \in \mathcal{C}_j}\right)$, where $F$ denotes the averaging operator and $\eta > 0$ is the stepsize,  the sample complexity improves proportionally to the number of systems in the cluster. 

In practice, however, the systems' cluster identities are typically unknown. To address this, a clustering step is introduced to group systems based on similarities in their locally estimated models. This clustered system identification approach proposed in \cite{toso2023learning} incurs an additional but exponentially decaying misclassification error, as clustering accuracy improves with the amount of collected data. The complete multitask adaptive control pipeline comprising data collection, clustering, multitask system identification, and control design is illustrated in Figure \ref{fig:multitask_control_pipeline}.

\subsection{Adversarial Systems}

Finally, we are interested in the setting where some systems may behave adversarially, transmitting corrupted or malicious model updates to the server during aggregation. To guarantee recoverability of the true cluster models, we assume an \emph{honest majority} within each cluster, i.e., $f_j < M_j/2$, where $M_j$ denotes the total number of systems in cluster $\mathcal{C}_j$, otherwise, reconstruction of the true model becomes information-theoretically impossible. In such cases, simple averaging of model updates fails, as a fraction of corrupted gradients can arbitrarily skew the aggregate and compromise convergence. 
To mitigate this, we leverage resilient aggregators that are widely studied in adversarial federated and distributed learning~\citep{guerraoui2024robust}.

\begin{definition}[$(f, \lambda)$-Resilient aggregation - Adapted from \cite{farhadkhani2022byzantine}] \label{def: resilient aggregation} Given a scalar $\lambda \geq 0$, an aggregation rule $F$ is said to be $(f, \lambda)$-resilient if, for any collection of matrices $G_1, \ldots, G_M \in \mathbb{R}^{d_x \times d'}$ and any subset of honest systems $\mathcal{H} \subseteq \{1, \ldots, M\}$ with $|\mathcal{H}| = m$,\vspace{-0.3cm}
$$
\left\|F\left(G_1, \ldots, G_M\right)-\bar{G}\right\| \leq \lambda \max _{i, j \in \mathcal H}\left\|G_i-G_j\right\|, \text{ with } \bar{G}:=\frac{1}{|\mathcal H|} \sum_{i \in \mathcal H} G_i.\\[-0.6cm]
$$
\end{definition}

\begin{remark}
$(f,\lambda)$-resilient aggregation ensures that the aggregated update remains close to the mean of the honest systems, up to a factor $\lambda$ capturing their dispersion. This guarantees robustness even when up to $f_j < M_j/2$ systems act adversarially. Many classical robust aggregation rules satisfy this property, including the \emph{coordinate-wise trimmed mean} (CWTM), \emph{coordinate-wise median} (CWMed), \emph{geometric median} (GM), \emph{minimum diameter averaging} (MDA), and \emph{mean-around-median} (MeaMed)~\citep{farhadkhani2022byzantine}. Typically, MDA and CWTM achieve $\lambda = \mathcal{O}(f_j/m_j)$, while GM and CWMed yield $\lambda = \mathcal{O}(1)$; the latter can be adjusted via pre-aggregation techniques such as nearest-neighbor mixing \citep{allouah2023fixing} or bucketing~\citep{ karimireddy2020byzantine}.\\[-0.8cm]
\end{remark}

\begin{remark} We assume that each cluster satisfies the honest-majority condition, i.e., $f_j < M_j/2$ for all $j \in [N_c]$. This standard assumption in robust aggregation~\citep{farhadkhani2022byzantine} guarantees convergence to the true cluster model even when a subset of systems transmit corrupted updates. Extending the framework to cases where adversarial systems can also misreport their cluster identities, thereby violating the honest-majority assumption, is left for future work. Promising directions include (i) randomized warm-up clustering steps to prevent adversarial concentration within specific clusters and (ii) robust clustering methods such as iterative filtering~\citep{ghosh2019robust}.
\end{remark}
\begin{goalbox}
\textbf{Goal.} The aim of adversarially robust multitask adaptive control is to cluster systems and 
design controllers that remain performant under heterogeneity and adversarial behavior.
\end{goalbox}

Our analysis quantifies the benefit of collaboration and the effects arising from clustering errors, intra-cluster heterogeneity, and adversarial behavior. We evaluate performance using the standard notion of cumulative regret from online learning \citep{abbasi2011regret}, defined as the difference between the cumulative cost incurred by the adaptive controller and that of the optimal controller under the true model: $\mathcal{R}_T^{(i)} =
\sum_{t=1}^{T} \big(c_t^{(i)} - J^{(i)}(K_\star^{(i)})\big)$,
where $c_t^{(i)}$ denotes the immediate cost for playing a suboptimal controller $\hat{K}^{(i)}$.

\vspace{-0.3cm}
\begin{algorithm}[H]
 \caption{Multitask Certainty-Equivalent Control} 
 \label{alg: CE with online clustering}
\begin{algorithmic}[1]
\State \textbf{Initialize: } $\hat K_1^{(i)} \gets K_0^{(i)}$ $\forall i \in [M]$,  $T \gets \tau_1 2^{k_{\fin}-1}$
\State \textbf{for} $k=1,2, \dots, k_{\fin}$
            \State \quad \textbf{for all systems} $i=1,\dots, M$ \textbf{(in parallel)}   \textcolor{blue}{\texttt{// Data collection}}
           \State \quad \quad  \textbf{for} $t = \tau_{k-1}, \tau_{k-1} + 1, \dots, \tau_{k}$
            \State \quad \quad \quad \textbf{If} $\snorm{x_t^{(i)}}^2 \geq x_b^2 \log T$ or $\snorm{\hat K_k^{(i)}}\geq K_b$\label{line: uniform bound}
                \State \quad \quad \quad \quad \textbf{Abort} and play $K_0^{(i)}$ forever
            \State \quad \quad \quad Play $u_t^{(i)} = \hat K_k^{(i)} x_t^{(i)} + \sigma_k g_t^{(i)}$ \textcolor{blue}{\texttt{// Exploration}}
        \State \quad \quad  \textbf{end for}
    \State \quad  \textbf{end for}
        \State \quad  $\widehat{\Theta}^{(1:M)}_k \leftarrow \texttt{RCSI}(\widehat{\Theta}^{(1:M)}_{k-1}, \mathcal{D}^{(1:M)}_{\tau_k}, N_c, N, \eta)$ \textcolor{blue}{\texttt{// Robust SysID}}
         \State \quad  \textbf{Update} the controller: $\hat K_{k+1}^{(i)} \gets K(\widehat{\Theta}^{(1:M)}_k), \forall i \in [M]$ \textcolor{blue}{\texttt{// Controller update}}
    \State \quad $\tau_{k+1} \gets 2 \tau_k$ 
\State \textbf{end for}
\end{algorithmic}
\end{algorithm}

\vspace{-0.7cm}
\begin{algorithm}[H]
\caption{\texttt{RCSI}: Robust Clustered System Identification} \label{algorithm: clustered sysid}
\begin{algorithmic}[1]
\State \textbf{for} $n = 0, 1, \ldots, N-1$ \textbf{do}
\State \quad \quad \text{Systems receive \textbf{all} the current cluster estimated models} $\{\widehat{\Theta}_j\}_{j \in [N_c]}$
\State \quad \quad \quad \textbf{for} $i = 1,\ldots,M$ \textbf{(in parallel)} \textbf{do} \textcolor{blue}{\texttt{// Cluster identity estimation}}
\State \quad \quad \quad \quad $\hat{j}=\argmin_{j \in [N_c]} \|{X}^{(i)} - \widehat{\Theta}_j {Z}^{(i)}\|^{2}$
\State \quad \quad \quad  \textbf{end for}
\State \quad \quad Construct the \textbf{cluster identity} set $\mathcal{C}_{\hat{j}}$
\State \quad \quad $G_{\ell}(\widehat{\Theta}^{(i)}) \leftarrow (X^{(\ell)} - \widehat{\Theta}^{(i)}Z^{(\ell)})Z^{(\ell)\top} \textcolor{black}{(Z^{(\ell)}Z^{(\ell)\top})^{-1}}, \forall \ell \in \mathcal{C}_{\hat{j}}$ \textcolor{blue}{\texttt{// Update}}
\State \quad \quad \quad \textbf{for} $i = 1,\ldots, M$ \textbf{(in parallel) do} \textcolor{blue}{\texttt{// Model estimation}}
\State \quad \quad  \quad \quad   \begin{small} $\widehat{\Theta}^{(i)} \leftarrow \widehat{\Theta}^{(i)} + \eta \textcolor{black}{F\left(\left\{G_{\ell}(\widehat{\Theta}^{(i)})\right\}_{\ell \in \mathcal C_{\hat{j}}}\right)}$ \textcolor{blue}{\texttt{// Aggregation}}\end{small} 
\State \quad \quad \quad \textbf{end for}
\State \textbf{end for}
\end{algorithmic}
\end{algorithm}

\vspace{-0.4cm}
Algorithm~\ref{alg: CE with online clustering} summarizes the adversarially robust multitask adaptive control procedure. Each system alternates between three phases: (i) data collection under the current controller, (ii) clustered robust system identification via resilient aggregation, and (iii) controller update using the estimated model. At each epoch $k$, all systems collect local trajectory data $\mathcal{D}_{\tau_k}^{(i)}$ in parallel. If the state or controller norm exceeds a threshold (i.e., $x_b$ and $K_b$), the system switches to a fallback stabilizing controller $K_0^{(i)}$ to ensure bounded trajectories and regret. The identification step, implemented by \texttt{RCSI} (Algorithm~\ref{algorithm: clustered sysid}), jointly performs clustering and robust aggregation: systems estimate their cluster identities, send local gradients $G_\ell(\widehat{\Theta}^{(i)})$, and the server aggregates them via a resilient rule $F$. The estimated model then defines the CE controller $\hat{K}_{k+1}^{(i)} = K(\hat{A}_k^{(i)}, \hat{B}_k^{(i)})$ for the next iteration.

We emphasize that a stabilizing initial controller $K_0^{(i)}$ is required for each system to ensure the initial trajectories remain bounded. Such controllers can be obtained using data-driven stabilization methods \citep{perdomo2021stabilizing, toso2025learning}. Moreover, Algorithm~\ref{alg: CE with online clustering} operates under the following two assumptions. First, we assume that the initial model estimate is sufficiently accurate to ensure consistent clustering~\citep{toso2023learning}.

\begin{assume} \label{assumption: initial model estimate} The initial model satisfies $\|\widehat{\Theta}^{(i)}_0 - {\Theta}^{(i)}_\star\| \leq C_\alpha \Delta_{\text{min}}$, for some constant $C_\alpha \in (0,\frac{1}{2})$.
\end{assume}

Second, we impose the following bounds on the state and controller norms.

\begin{assume}\label{eq: state_controller_norm} We assume that the state and controller bounds are
$$
x_b \geq 400\left(P_0^{\vee}\right)^2 \Psi^{\vee}_{B} \sigma_w \sqrt{d^\prime}, \text{ and } K_b \geq \sqrt{P_0^{\vee}}.\\[-0.3cm]
$$
where $ \Psi^{(i)}_{B} \triangleq \max \left\{1,\|B_{\star}^{(i)}\|\right\}, \; 
\Psi^{\vee}_{B} \triangleq \max_{i} \Psi^{(i)}_{B}, \;
P_0^{\vee} \triangleq \max_{i}\left\|P_{K_0^{(i)}}^{(i)}\right\|,$ with $P_{K}^{(i)}$ denoting the solution to the
discrete Lyapunov equation, $P_K^{(i)} \triangleq \texttt{dlyap}(A_\star^{(i)} + B_\star^{(i)} K, Q + K^\top R K )$
\end{assume}

We present our theoretical analysis progressively across three settings, building from the simplest to the most general case:  (i) intra-cluster homogeneity ($f = 0$, $\epsilon_{\mathrm{het}} = 0$), (ii) bounded heterogeneity ($f = 0$, $\epsilon_{\mathrm{het}} > 0$),  and (iii) adversarial systems ($f > 0$, $\epsilon_{\mathrm{het}} > 0$). For each setting, we first establish estimation error bounds and subsequently the corresponding regret bounds.

\section{Error Bounds for Multitask System Identification}
\label{sec: sysID}

We are now ready to establish error bounds for the clustered multitask identification procedure described in Algorithm~\ref{algorithm: clustered sysid}. This section characterizes how collaboration, heterogeneity, and adversarial behavior influence the estimation error of the local system models. Complete proofs and constant definitions are provided in Appendix~\ref{appendix: multitask sysid}. For clarity of presentation, we introduce the following key quantities: $r^\vee \triangleq \max_{t \in [T],\ell \in \mathcal{C}_j} {\operatorname{tr}\left(\Sigma^{(\ell)}_{z_t}\right)}/{\|\Sigma^{(\ell)}_{z_t}\|}, P_{\star}^{\wedge} \triangleq \min _{i \in [M]}\|P_{\star}^{(i)}\|,$ $C_\tau = \texttt{poly}(d_x,d_u)$, and the contraction rate $\rho = 1 -\eta$, for some step-size $\eta \in (0,1)$.

\begin{lemma}[Intra-cluster homogeneity]
\label{lemma: sysID homogeneous}
Consider a cluster $\mathcal C_j$ composed of $M_j$ identical systems. Suppose that Assumption \ref{assumption: initial model estimate} holds and that the initial epoch length is set sufficiently large,  such that \vspace{-0.1cm}
$$\tau_1 \geq C_\tau \max\left\{\frac{d^\prime\Delta^2_{\min}\sigma_w \log\left(\frac{(2d_x + d_u)M_j}{\delta}\right)}{dM_j}, r^\vee + \log\left(\frac{2M_j}{\delta}\right) , \log(4T^2), \frac{\log \frac{1}{P_{\star}^{\wedge}}}{\log\left(1- \frac{1}{ P_{\star}^{\wedge}}\right)} \right\},\\[-0.3cm]$$ 
for a small $\delta \in (0,1)$. Then, after $N \geq \log\left(C_\alpha \Delta_{\text{min}}\tau^2_1\right)/\log(1/\rho)$ iterations of \texttt{RCSI}, with probability $1-\delta$, the model estimate $\widehat{\Theta}^{(i)}$ at epoch $k$ satisfies:
$$
\left\|\widehat{\Theta}^{(i)}_k - \Theta_\star^{(i)}\right\|_F^2 \leq
C_{\mathrm{stat}}\,
\frac{\sigma_w^2 (d_x^2 + d_x d_u)\log(M_j/\delta)}{\sigma_k^2 M_j \tau_k}
+
C_{\mathrm{mis},1}\,e^{-C_{\mathrm{mis},2}\sigma_k^2\tau_k}.
$$
\end{lemma}

\allowdisplaybreaks

\begin{proposition}[Intra-cluster heterogeneity]
\label{prop: sysID heterogeneous}
Consider a cluster $\mathcal C_j$ consisting of $M_j$ similar but non-identical systems with bounded intra-cluster heterogeneity as in Assumption~\ref{assumption: heterogeneity}. Suppose the initial epoch length $\tau_1$ and the number of iterations $N$ in \texttt{RCSI} are chosen as in Lemma~\ref{lemma: sysID homogeneous}. Then, for a small $\delta \in (0,1)$, the model estimate $\widehat{\Theta}^{(i)}$ at epoch $k$ satisfies:
$$
\left\|\widehat{\Theta}^{(i)}_k - \Theta_\star^{(i)}\right\|_F^2
\leq
C_{\mathrm{stat}}
\frac{\sigma_w^2 (d_x^2 + d_x d_u)\log(M_j/\delta)}{\sigma_k^2 M_j \tau_k}
+
C_{\mathrm{het}} \epsilon_{\mathrm{het}}^2
+
C_{\mathrm{mis},1} e^{-C_{\mathrm{mis},2}\sigma_k^2\tau_k},
$$
with probability at least $1-\delta$.
\end{proposition}

\begin{lemma}[Adversarial systems]
\label{lemma: sysID adversarial}
Consider a cluster $\mathcal{C}_j$ containing $M_j$ similar but non-identical systems, among which at most $f_j < M_j/2$ are adversarial and $m_j = M_j - f_j$ are honest. 
Assume the aggregation rule is $(f_j,\lambda)$-resilient as defined in Definition~\ref{def: resilient aggregation}, and that Assumption \ref{assumption: heterogeneity} holds. 
Suppose further that the initial epoch length $\tau_1$ and the number of iterations $N$ in \texttt{RCSI} are chosen as in Lemma~\ref{lemma: sysID homogeneous}. Then, for a small $\delta \in (0,1)$, the model estimate $\widehat{\Theta}^{(i)}$ at epoch $k$ satisfies:

\begin{align*}
 \left\|\widehat{\Theta}^{(i)}_k - \Theta_\star^{(i)}\right\|_F^2
&\leq
C_{\mathrm{stat}}
\frac{\sigma_w^2 (d_x^2 + d_x d_u)\log(M_j/\delta)}{\sigma_k^2\tau_k}
\left(\frac{1}{m_j}+\lambda^2 d_x\right)
+C_{\mathrm{het}}(1+\lambda)^2\epsilon_{\mathrm{het}}^2\\
&+C_{\mathrm{mis},1}e^{-C_{\mathrm{mis},2}\sigma_k^2\tau_k}, \text{ w.p. } 1-\delta. 
\end{align*}
\end{lemma}

\noindent\textbf{Discussion.}
Taken together, Lemma \ref{lemma: sysID homogeneous}, Proposition \ref{prop: sysID heterogeneous}, and
Lemma \ref{lemma: sysID adversarial} provide a unified characterization of the error bounds for the multitask system identification step under adversarial and heterogeneous systems. Under intra-cluster homogeneity, the estimation error scales inversely with both the number of systems and the data length (epoch size), achieving $\mathcal{O}(\sigma^2_w/(M_j\tau_k))$ consistency. Introducing heterogeneity adds a fixed bias proportional to $\epsilon_{\mathrm{het}}^2$, while adversarial systems introduce an additional $\lambda$-scaled term determined by the resilience coefficient of the aggregation rule. The proofs of these results are provided in Appendix \ref{appendix: multitask sysid}, where the estimation error is controlled using the matrix Hoeffding inequality \citep{tropp2011user}. These results underpin the regret analysis in Section~\ref{sec: regret analysis}, where improved estimation accuracy translates into a reduction in the expected regret.

\section{Regret Analysis}
\label{sec: regret analysis}

We are now in place to establish the regret bounds for Algorithm~\ref{alg: CE with online clustering} under the three settings of interest: (i) intra-cluster homogeneity, (ii) intra-cluster heterogeneity, and (iii) the presence of adversarial systems. We quantify how clustering accuracy, intra-cluster similarity, and adversarial robustness jointly affect the regret of the CE controller for any honest system within a given cluster $\mathcal{C}_j$. Complete derivations are provided in Appendix~\ref{appendix: regret analysis}. To clarify exposition, we introduce the following quantities: $\Omega_1 \triangleq 142C_{\mathrm{stat}} \|P^{(i)}_{\star}\|^{8}\sigma^2_w \log((M_j T)/\delta) + 2 d_u (1 + 2\|P^{(i)}_{\star}\| \Psi^{(i)2}_{B}), \Omega_2 \triangleq 3d_x\left\|P^{(i)}_{K^{(i)}_0}\right\|\Psi^{(i)2}_{B} + 2x_b^2\left\|P^{(i)}_{\star}\right\|, \Omega_3 \triangleq 142C_{\mathrm{mis,1}}\|P^{(i)}_{\star}\|^{8},$ and $\Omega_4 \triangleq 142 C_{\mathrm{het}}\|P^{(i)}_{\star}\|^8.$

\begin{theorem}[Intra-cluster homogeneity]
\label{theorem: regret homogeneous}
Fix a system $i$ belonging to a homogeneous cluster $\mathcal{C}_j$ of size $M_j$. Let Assumption \ref{assumption: initial model estimate} hold and consider the doubling-epoch schedule $\tau_k = 2^{k-1}\tau_1$
with initial epoch length $\tau_1$ and number of \texttt{RCSI} iterations $N$ as in Lemma \ref{lemma: sysID homogeneous}, with exploration sequence
$\sigma_k^2 =  \tfrac{\sqrt{d}}{d_x^2 + d_x d_u}\tfrac{1}{\sqrt{\tau_k M_j}}$. Then the expected regret of any system $i \in \mathcal{C}_j$ under Algorithm~\ref{algorithm: clustered sysid} satisfies:
\begin{align*}
\mathbb{E} \left[\mathcal{R}_T^{(i)}\right]
\leq 
\Omega_1 \sqrt{\tfrac{dT}{M_j}}
+
\Omega_2 (\log T)^2
+
\Omega_3 T e^{-C_{\mathrm{mis},2}\sqrt{\tfrac{\tau_1}{M_j}}}.
\end{align*}
\end{theorem}
\begin{corollary}[Intra-cluster heterogeneity]
\label{corollary: regret heterogeneous}
Consider a system $i$ belonging to a heterogeneous cluster $\mathcal{C}_j$ of size $M_j$. Suppose Assumptions \ref{assumption: heterogeneity} and \ref{assumption: initial model estimate} hold. 
Then, under the same conditions on the epoch length $\tau_1$, number of \texttt{RCSI} iterations $N$, and exploration sequence $\sigma^2_k$ in
Theorem~\ref{theorem: regret homogeneous}, the expected regret for any system $i \in \mathcal{C}_j$ satisfies:
\begin{align*}
\mathbb{E}\left[\mathcal{R}_T^{(i)}\right]
\leq
\Omega_1 \sqrt{\tfrac{dT}{M_j}}
+
\Omega_2 (\log T)^2
+
\Omega_3 T e^{-C_{\mathrm{mis},2}\sqrt{\tfrac{\tau_1}{M_j}}}
+
\Omega_4 T \epsilon_{\mathrm{het}}^2.
\end{align*}
\end{corollary}
\begin{theorem}[Adversarial systems]
\label{theorem: regret adversarial}
Let each cluster $\mathcal{C}_j$ satisfy honest-majority $f_j < \frac{M_j}{2}$. Suppose that the aggregation rule is $(f_j,\lambda)$-resilient as defined in Definition~\ref{def: resilient aggregation}.
Under the conditions on the epoch length $\tau_1$ and number of \texttt{RCSI} iterations $N$ as in Lemma \ref{lemma: sysID homogeneous}, with exploration sequence $\sigma_k^2 = \frac{\sqrt{d^2_u d_x}}{d_x^2 + d_xd_u}\sqrt{\frac{1 + \lambda^2 d_x m_{{j}}}{\tau_k m_{{j}}}}$. Then, for any honest system in $\mathcal{C}_j$, the expected regret satisfies:
\begin{align*}
\mathbb{E} \left[\mathcal{R}_T^{(i)}\right]
\leq
\Omega_1
\sqrt{\tfrac{dT(1+\lambda^2 d_x m_j)}{m_j}}
+
\Omega_2 (\log T)^2
+
\Omega_3 T\,e^{-C_{\mathrm{mis},2}\sqrt{\tfrac{(1+\lambda^2 d_x m_j)\tau_1}{m_j}}}
+
\Omega_4 T (1+\lambda)^2 \epsilon_{\mathrm{het}}^2.
\end{align*}
\end{theorem}

\noindent \textbf{Discussion.} Theorem~\ref{theorem: regret homogeneous} demonstrates that, under intra-cluster homogeneity, multitask adaptive control achieves an $\mathcal{O}(1/\sqrt{M_j})$ improvement in the leading term of the regret compared to the single-system case~\citep{simchowitz2020naive}. This scaling confirms that jointly estimating a common model for identical systems effectively reduces the regret of adaptive LQR control. Corollary~\ref{corollary: regret heterogeneous} extends this result to heterogeneous clusters, where a residual bias term proportional to $\epsilon_{\mathrm{het}}^2$ appears in the regret. This term reflects the unavoidable model mismatch within each cluster, introducing an additional component linear in $T$. When $\epsilon_{\mathrm{het}}$ is sufficiently small, the regret remains dominated by the $\widetilde{\mathcal{O}}(\sqrt{T/M_j})$ term, thus preserving the benefit of collaboration.

Finally, Theorem~\ref{theorem: regret adversarial} demonstrates that the collaboration benefit persists even in the presence of adversarial systems, provided the honest-majority condition holds and the aggregation rule is $(f_j,\lambda)$-resilient. The additional factor $(1+\lambda^2 d_x m_j)$ quantifies the robustness cost introduced by resilient aggregation. As discussed earlier, the resilience coefficient $\lambda$ typically scales as $\mathcal{O}(f_j/m_j)$, thereby preserving the regret reduction under adversarial systems. Taken together, these results highlight the robustness of multitask adaptive control under heterogeneous and adversarial systems.

\subsection{Proof Idea}

The proof of Theorem~\ref{theorem: regret adversarial} follows a standard regret decomposition~\citep{cassel2020logarithmic} into three parts: (i) regret under the success event, when Algorithm~\ref{alg: CE with online clustering} does not abort and Lemma~\ref{lemma: sysID adversarial} holds; (ii) regret under failure; and (iii) regret from the first epoch. We prove that the success event occurs with high probability, at least $1 - T^{-2}$ (see Appendix~\ref{appendix: probability of success}), ensuring that (i) dominates.  

Under success, the regret scales as $\mathcal{O}\big((\tau_k - \tau_{k-1})(\|\widehat\Theta^{(i)}_k - \Theta^{(i)}_\star\|_F^2 + \sigma_k^2)\big)$, where estimation errors propagate via the Lyapunov bound on $P^{(i)}_{\star}$, linking model inaccuracy to cost difference~\citep[Theorem~3]{simchowitz2020naive}. Balancing exploration ($\sigma_k$) and exploitation ($\tau_k$) in the error bound of Lemma~\ref{lemma: sysID adversarial}, and summing over exponentially growing epochs, yields the dominant $\widetilde{\mathcal{O}}(\sqrt{T/m_j})$ term, with additive effects from resilient aggregation, heterogeneity, and clustering misclassification. The $\mathcal{O}((\log T)^2)$ term arises from the initial epoch. The regret under failure is negligible.

\section{Numerical Validation} \label{sec:numerics}

We consider multiple unicycle robots with details on the dynamics and implementation\footnote{Code available at \url{https://github.com/jd-anderson/multi_task_adaptive_control}} deferred to Appendix \ref{appendix: additional numerics}.  Figure~\ref{fig:six_panel_results} summarizes the performance of Algorithm~\ref{alg: CE with online clustering} across six panels, highlighting the effects of collaboration, heterogeneity, clustering accuracy, and adversarial behavior. In the top-left panel, homogeneous clusters show decreasing regret with more systems per cluster, matching the $\widetilde{\mathcal{O}}(\sqrt{T/M_j})$ scaling in Theorem~\ref{theorem: regret homogeneous}. The top-middle panel introduces intra-cluster heterogeneity, where regret increases due to the $\mathcal{O}(\epsilon_{\mathrm{het}}^2 T)$ term (Corollary~\ref{corollary: regret heterogeneous}); estimation accuracy initially improves with collaboration but saturates at a bias floor proportional to $\epsilon_{\mathrm{het}}^2$.

The top-right panel reports clustering accuracy, showing exponentially decaying misclassification with more samples and faster convergence for larger clusters, consistent with the $e^{-C_{\mathrm{mis},2}\sigma_k^2\tau_k}$ rate. The bottom-left panel presents the adversarial setting, where a fraction $\rho_{\mathrm{byz}} = f_j/M_j$ of systems send corrupted updates (see Appendix~\ref{appendix: additional numerics}). Increasing $\rho_{\mathrm{byz}} \in \{10\%, 20\%, 30\%\}$ preserves sublinear regret for $\rho_{\mathrm{byz}} < 0.5$, in line with the $\mathcal{O}(1 + \lambda^2 d_x)$ inflation predicted by Theorem~\ref{theorem: regret adversarial}.

The bottom-middle panel compares two resilient aggregation rules under identical adversarial contamination. We evaluate the trimmed mean and geometric median aggregators, which differ in their resilience coefficients~$\lambda$. Both converge, confirming the $(f_j,\lambda)$-resilience property; consistent with~\cite{allouah2023fixing}, the trimmed mean achieves slightly lower regret due to a smaller~$\lambda$. A broader empirical study across alternative aggregation schemes is left for future work.

The bottom-right panel compares Algorithm~\ref{alg: CE with online clustering} with the multitask representation learning method of~\cite{lee2024regret} (denoted here by \texttt{RepL}). We evaluate both under (i) structurally homogeneous systems, where a shared representation exists, and (ii) a heterogeneous configuration with an added system violating this assumption. While \texttt{RepL} performs well in case (i), it deteriorates sharply without a shared model representation (ii). In contrast, Algorithm \ref{alg: CE with online clustering}, though affected by intra-cluster heterogeneity, confines learning to similar systems and achieves lower regret under heterogeneity.

\begin{figure}
  \centering
  \setlength{\tabcolsep}{2pt}
  \renewcommand{\arraystretch}{0}

  \begin{tabular}{ccc}
    \includegraphics[width=0.3\linewidth]{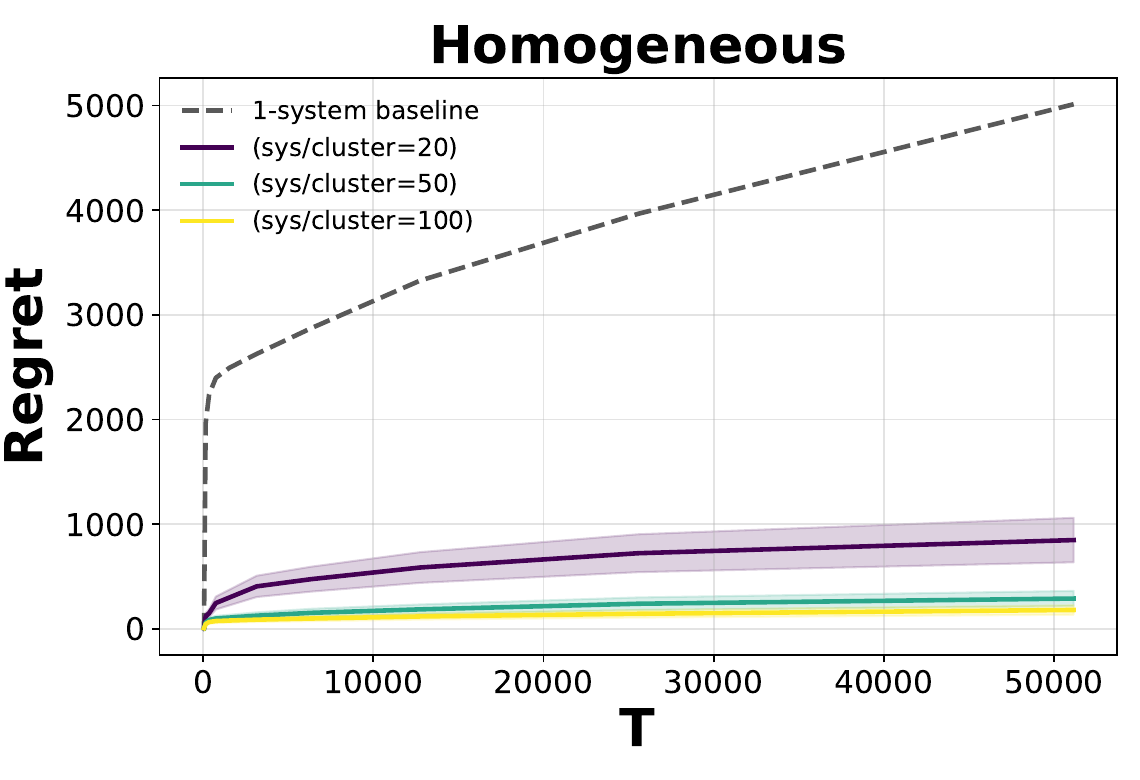} &
    \includegraphics[width=0.3\linewidth]{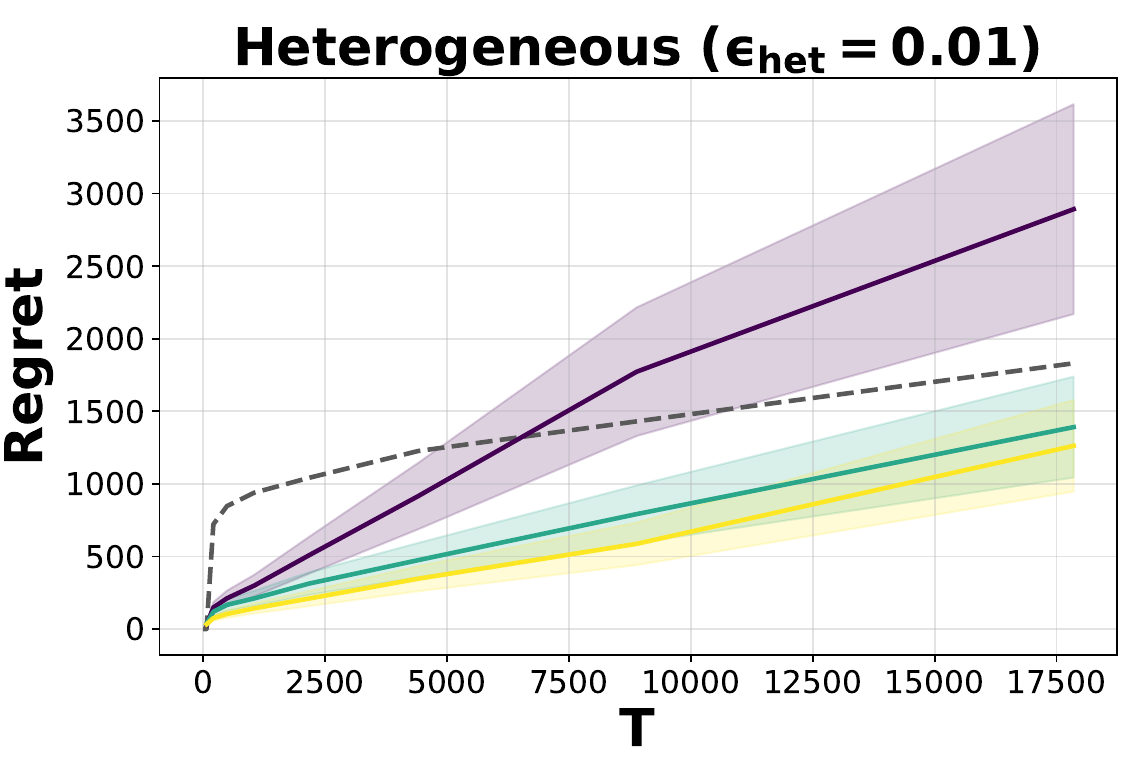} &
    \includegraphics[width=0.3\linewidth]{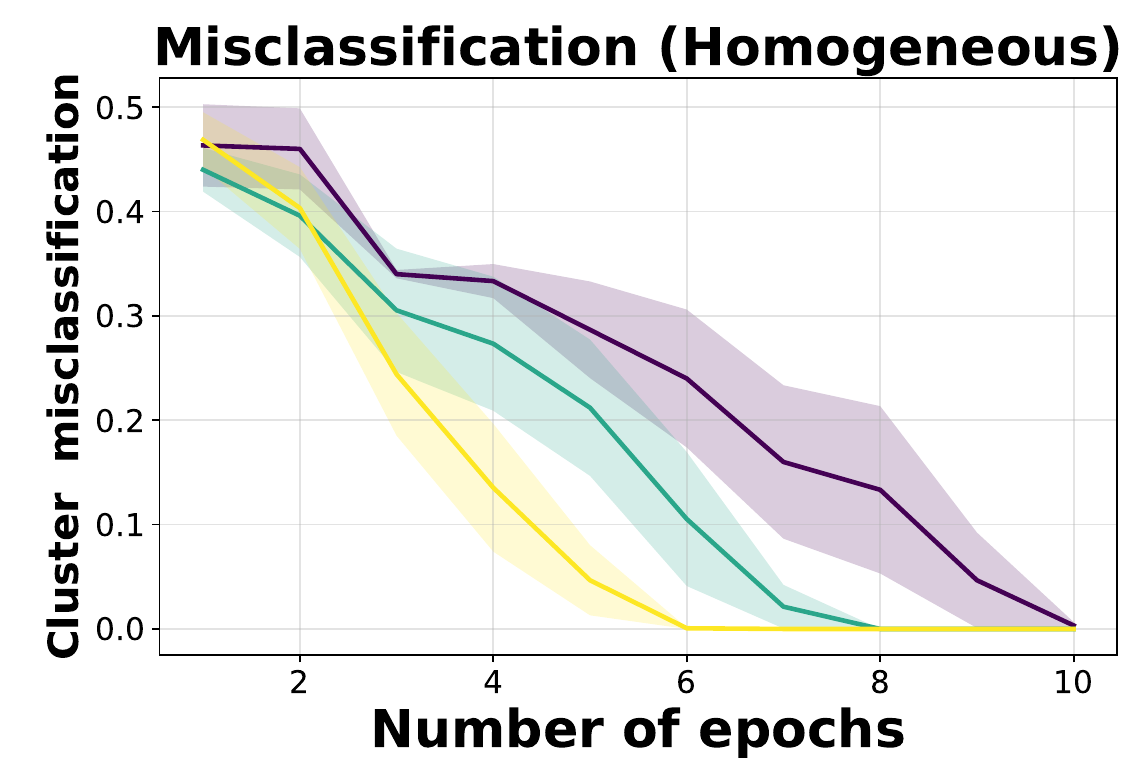} \\
    \includegraphics[width=0.3\linewidth]{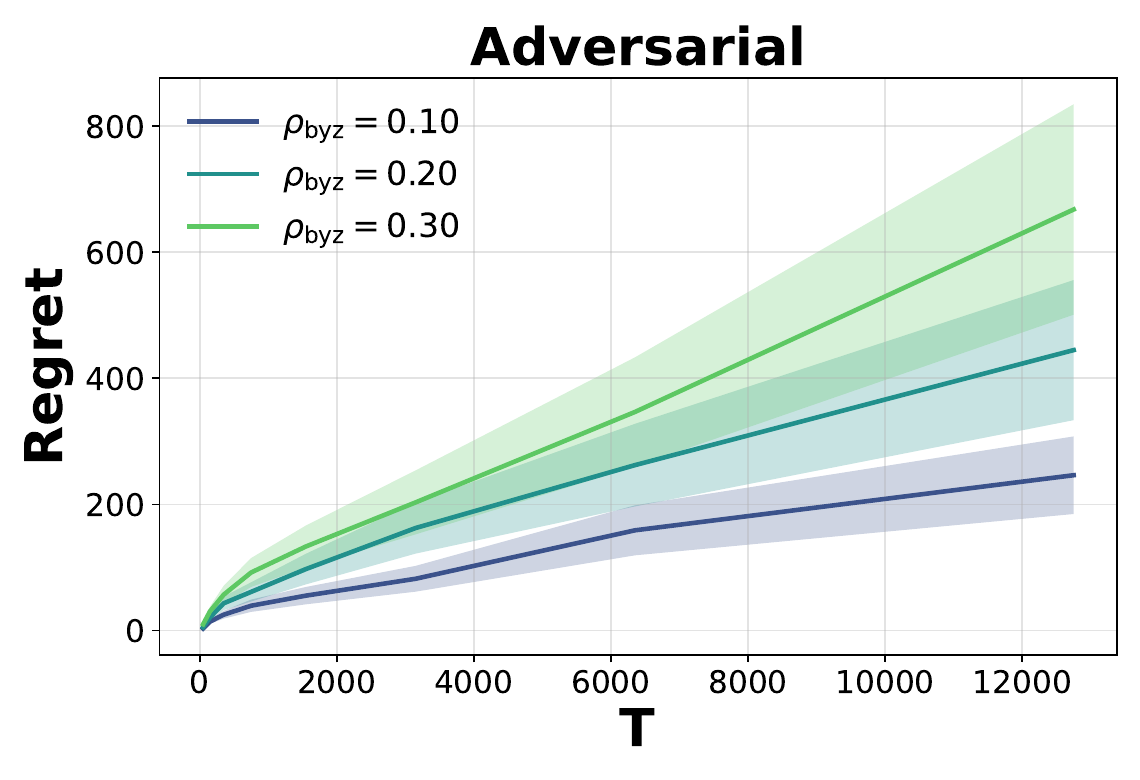} &
    \includegraphics[width=0.3\linewidth]{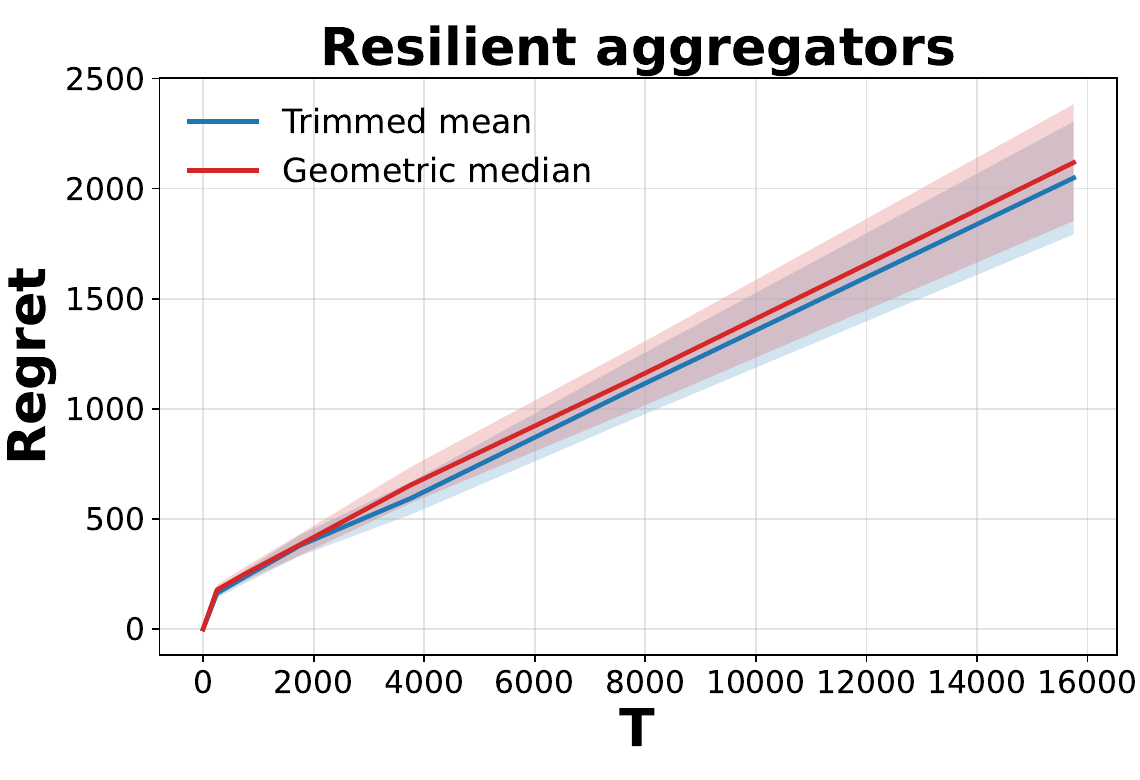} &
    \includegraphics[width=0.3\linewidth]{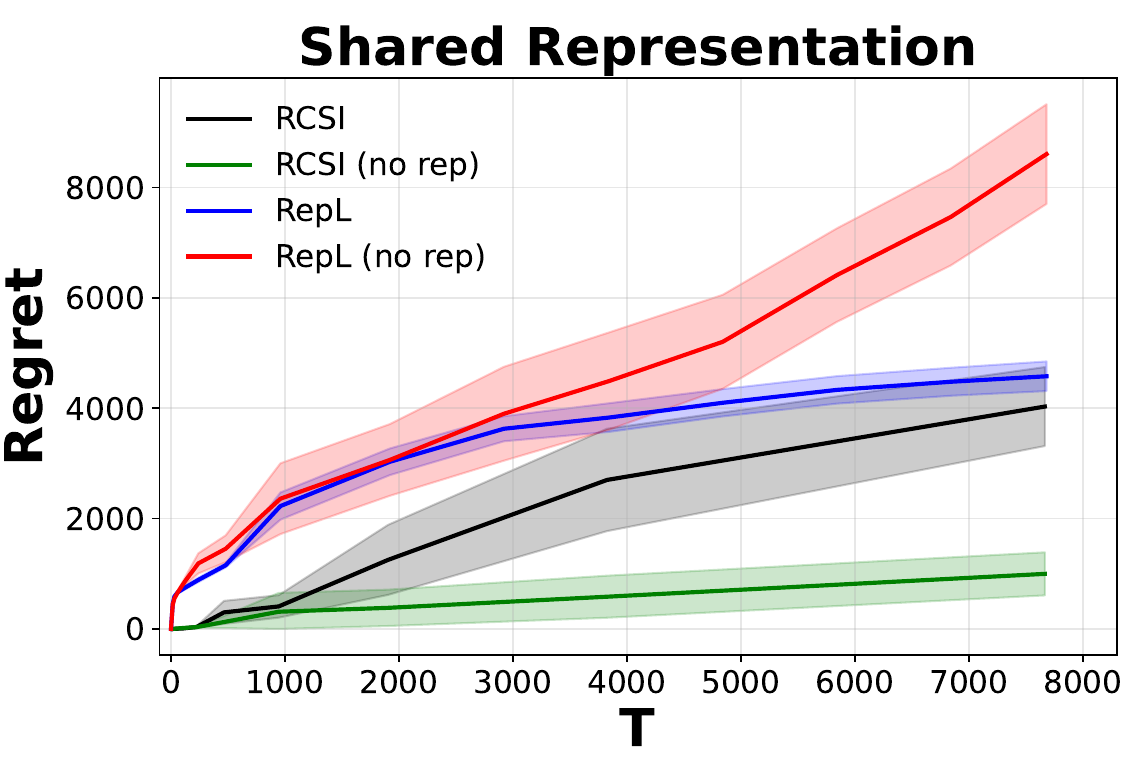} \\
  \end{tabular}
  \caption{Illustration of main results. Top row--(left) homogeneous clusters; (middle) intra-cluster heterogeneity; (right) misclassification rate. Bottom row--(left) adversarial systems; (middle) robust aggregation comparison; (right) comparison with multitask representation learning.}
  \label{fig:six_panel_results}
\end{figure}

\section{Future Work}

This work provides the first analysis of multitask adaptive control under heterogeneous and adversarial systems, opening several promising directions. Our regret bounds reveal a non-vanishing heterogeneity bias $\epsilon_{\mathrm{het}}$, scaling as $\mathcal{O}(\epsilon_{\mathrm{het}}^2 T)$, which may limit collaboration when intra-cluster heterogeneity is large. A natural extension is to integrate representation learning within clusters to mitigate this bias and remove the interplay between heterogeneity and adversarial effects $\mathcal{O}(\lambda^2 \epsilon_{\mathrm{het}}^2)$. Another direction is developing robust clustering methods to handle cases where adversarial systems corrupt cluster identities. Extending this work to nonlinear dynamics and non-quadratic objectives also remains an important avenue for future work.

\section{Acknowledgments} 
Leonardo F. Toso thanks Rafael Pinot and Nirupam Gupta for instructive discussions on adversarial machine learning.  Leonardo F. Toso is funded by the Center for AI and Responsible Financial Innovation (CAIRFI) Fellowship and by the Columbia Presidential Fellowship. James Anderson is partially funded by NSF grants ECCS 2144634 and 2231350 and the Columbia Center of AI Technology in collaboration with Amazon.

\bibliographystyle{unsrtnat}
\bibliography{references}

\newpage 
\appendix
\addcontentsline{toc}{section}{Appendix}
\part{Appendix}
\parttoc 
\newpage

\section{Appendix Roadmap} The appendix is organized as follows.   Appendix~\ref{appendix: additional numerics} provides additional experiments and details on the experimental setup used in Section~\ref{sec:numerics} to validate and illustrate our theoretical guarantees. Appendix~\ref{appendix: regret decomposition} presents the detailed regret decomposition into three components: (i) the contribution corresponding to the success of Algorithm~\ref{alg: CE with online clustering} (i.e., no abortion), (ii) the contribution associated with its failure, and (iii) the term arising from initialization with a suboptimal stabilizing controller. In Appendix~\ref{appendix: auxiliary results}, we revisit key auxiliary results, including matrix concentration inequalities from~\cite{tropp2011user}, regret bounds from~\cite{lee2023nonasymptotic}, and the assumption that the state-input covariates are distributed according to a $\beta$-mixing stationary processes, we refer the reader to \cite{yu1994rates} for the definition of such geometric processes. Appendix~\ref{appendix: multitask sysid} presents error bounds for multitask system identification under three settings: (i) intra-cluster homogeneity, (ii) intra-cluster heterogeneity, and (iii) adversarial systems. Appendix~\ref{appendix: probability of success} establishes that the success event under which Algorithm~\ref{alg: CE with online clustering} does not abort holds with high probability. Finally, Appendix~\ref{appendix: regret analysis} provides the regret analysis, leveraging the estimation error bounds and success-event probability to derive the regret bounds.

\section{Numerical Implementation Details}

\label{appendix: additional numerics}

\noindent \textbf{System dynamics.} All experiments were conducted on discrete and linear time-invariant (LTI) systems of the form
\begin{equation*}
    x^{(i)}_{t+1} = A^{(i)}_\star x^{(i)}_t + B^{(i)}_\star u^{(i)}_t + w^{(i)}_t, \text{ with noise } w^{(i)}_t \sim \mathcal{N}(0, \sigma_w^2 I),
\end{equation*}
where $x_t \in \mathbb{R}^3$ is the system state, $u_t \in \mathbb{R}^2$ is the control input, and $(A^{(i)}_\star, B^{(i)}_\star)$ denote the dynamics of the $i$-th subsystem. Each subsystem is derived from a discrete-time unicycle model with kinematics
$
\dot{p}^{(i)}_x = v^{(i)} \cos\theta^{(i)}, \;
\dot{p}^{(i)}_y = v^{(i)} \sin\theta^{(i)}, \;
\dot{\theta}^{(i)} = \omega^{(i)},
$
where the state is $x^{(i)} = (p^{(i)}_x, p^{(i)}_y, \theta^{(i)})$ and the control input is $u^{(i)} = (v^{(i)}, \omega^{(i)})$. Here, $(p^{(i)}_x,p^{(i)}_y)$ denotes the position, $\theta^{(i)}$ the orientation, $v^{(i)}$ the forward velocity, and $\omega^{(i)}$ the yaw rate of the $i$-th robot. The dynamics are linearized around different nominal operating points $(v_{0,j}, \theta_{0,j}) \in \{(1.0, 0^\circ), (1.0, 45^\circ), (0.8, 90^\circ)\}$ and discretized (with Euler step $\Delta t = 0.1$) to obtain the cluster models:
$$
A_j =
\begin{bmatrix}
1 & 0 & -\Delta t v_{0,j} \sin \theta_{0,j} \\
0 & 1 &  \Delta t v_{0,j} \cos \theta_{0,j} \\
0 & 0 & 1
\end{bmatrix},
\quad
B_j =
\begin{bmatrix}
\Delta t \cos \theta_{0,j} & 0 \\
\Delta t \sin \theta_{0,j} & 0 \\
0 & \Delta t
\end{bmatrix},
$$ 
for each cluster $j \in \{1,2,3\}$, heterogeneous systems are created by adding Gaussian perturbations to the nominal parameters:
\begin{equation*}
    A^{(i)}_\star = A_j + \varepsilon_{ij}, \; B^{(i)}_\star = B_j + \delta_{ij},
\end{equation*}
where $\varepsilon_{ij}$ and $\delta_{ij}$ are elementwise Gaussian perturbations with zero mean and variance $\epsilon_{\mathrm{het}}$.  
Unless otherwise stated, $\epsilon_{\mathrm{het}} = 0.01$, which corresponds to a relatively strong heterogeneity level since the smallest inter-system coefficient difference among the nominal cluster models is on the order of $0.04$. The process noise variance was fixed to $\sigma_w^2 = 0.01$. We set $Q = I_{d_x}$ and $R = I_{d_u}$.\vspace{0.2cm}

\noindent \textbf{Experimental configurations.} In the first three panels (\emph{homogeneous}, \emph{heterogeneous}, and \emph{misclassification}) in Figure \ref{fig:six_panel_results}, we used the above three cluster models and their perturbed systems with $\epsilon_{\mathrm{het}} = 0.01$.  
The homogeneous case is obtained by setting $\epsilon_{\mathrm{het}} = 0$ for all systems.
In the the misclassification experiment we use the same configuration as the homogeneous case.

For the fourth experiment (adversarial setting), we considered the same system clusters and heterogeneity level ($\epsilon_{\mathrm{het}} = 0.01$) with $20$ systems per cluster.  Adversarial systems were randomly selected with corruption ratio $\rho_{\mathrm{byz}} \in \{0, 0.15, 0.35\}$ to evaluate robustness.

In the fifth experiment (aggregation rules comparison), the same cluster models were used with $50$ systems per cluster, and $\epsilon_{\mathrm{het}} = 0.01$, and adversarial ratio $\rho_{\mathrm{byz}} = 0.15$.  
Two aggregation rules were compared: (i) trimmed-mean with $\alpha_{\mathrm{trim}} = \rho_{\mathrm{byz}} + 0.05$, and (ii) geometric-median aggregation. Both methods were implemented using the \texttt{RCSI} update under identical simulation conditions.

The final experiment corresponds to the shared-representation setting.  Here, we used $26$ systems belonging to two clusters (see dataset files \texttt{A\_norep.pkl} and \texttt{B\_norep.pkl}).  
The first 25 systems share a common representation and were treated as belonging to a single cluster, while the $26$-th system was constructed to violate this assumption, representing a structurally distinct system.  We ran both Algorithm~\ref{alg: CE with online clustering} and the multitask representation learning method from~\cite{lee2024regret} (denoted here by \texttt{RepL}) using identical algorithm hyperparameters:
$$
\tau_1 = 15, \; k_{\mathrm{fin}} = 9, \; x_b = 25, \; K_b = 15, \text{ with } \epsilon_{\mathrm{het}} = 0.004.
$$
For both methods, the exploration sequence was geometrically decaying as
$$
    \sigma_{u,k} \in \{0.663, 0.411, 0.124, 0.037, 0.011, 0.0034, 0.0010, 0.0010\}.
$$

\vspace{0.2cm}
\noindent \textbf{Adversarial implementation.}
To simulate adversarial systems, each system $i$ is assigned a binary corruption flag with probability $\rho_{\mathrm{byz}}$.  
Corrupted systems replace their empirical regression statistic $(XZ_i, ZZ_i)$ with an adaptive convex combination:
\begin{equation}
    XZ_i^{(\mathrm{byz})} = (1 - \beta) XZ_i + \beta (\Theta_{j} ZZ_i) + \varepsilon,
\end{equation}
where $\Theta_{j}$ is the parameter matrix of an incorrect cluster, $\beta \in [0,1]$ controls the corruption strength ($\beta=0.6$), and $\varepsilon$ is small Gaussian perturbation ensuring numerical non-degeneracy.  
When $\rho_{\mathrm{byz}} = 0$, the algorithm defaults to \texttt{CSI} aggregation (simple averaging).  
When $\rho_{\mathrm{byz}} > 0$, robust aggregation uses either a trimmed-mean rule (\texttt{alpha\_trim}$=\rho_{\mathrm{byz}}+0.05$) or geometric-median consensus depending on the flag \texttt{use\_geom\_median}.

\section{Regret Decomposition}  \label{appendix: regret decomposition}

To synthesize our regret bounds, we begin by defining the events $\mathcal{E}_{\mathrm{success}}$ and $\mathcal{E}_{\mathrm{failure}}$. These events characterize the high-probability regimes under which the system trajectories remain bounded and the clustered system identification produces sufficiently accurate estimates of the dynamics. Each of the following settings--(i) intra-cluster homogeneity, (ii) intra-cluster heterogeneity, and (iii) adversarial systems--are associated with a distinct estimation event.\\

\noindent \textbf{Success and failure events.}  We first introduce the event which ensures that the state and designed controller norms remain uniformly bounded throughout the epoch length for all epochs:
\begin{align}\label{eq: Ebound}
    \mathcal{E}_{\mathrm{bound}} 
    := \Big\{ \|x_t\|^2 \leq x_b^2 \log T, \ \forall t = 1,\dots,T \Big\} 
    \cap \Big\{ \|{\hat K_k}\| \leq K_b, \ \forall k = 1,\dots,k_{\fin} \Big\}.
\end{align}

We recall that the true and estimated system models are denoted by
$$
\Theta^{(i)}_\star = [A^{(i)}_\star \;\ B^{(i)}_\star], 
\;\
\widehat{\Theta}^{(i)} = [\widehat{A}^{(i)} \;\ \widehat{B}^{(i)}],
$$
and introduce a the estimation events, each corresponding to one of the three principal cases considered in this work:\\

\noindent $\bullet$ \textbf{Case 1: Intra-cluster homogeneity.} When all systems within a cluster share identical dynamics, we denote by $\mathcal{E}_{\mathrm{est},1}^{(k)}$ the event which the estimation error admits the following bound:
\begin{align}\label{eq: Eest1}
    \left\| \widehat{\Theta}^{(i)}_k - \Theta^{(i)}_\star \right\|_F^2 
    \leq \frac{C_{\mathrm{stat}}\sigma^2_w (d^2_x + d_x d_u) \log(1/\delta)}{\sigma^2_k M_{j} \tau_{k}} 
    + C_{\mathrm{mis,1}}\exp\!\big(- C_{\mathrm{mis,2}}\sigma^2_k\tau_k\big),
\end{align}
for all epochs $k \in [k_{\fin}],$ and systems $i \in [m]$. The constants \(C_{\mathrm{stat}}, C_{\mathrm{mis,1}},C_{\mathrm{mis,2}}\) are positive and universal, and they are defined in Lemma~\ref{lemma: homogeneous-regression-error}, along with their derivations.\\

\noindent $\bullet$ \textbf{Case 2: Intra-cluster heterogeneity.} When systems within a cluster are similar but not identical with bounded heterogeneity $\epsilon_{\mathrm{het}}$, we denote by $\mathcal{E}_{\mathrm{est},2}^{(k)}$ the event when the estimation error satisfies:
\begin{align}\label{eq: Eest2} 
    \left\| \widehat{\Theta}^{(i)}_k - \Theta^{(i)}_\star  \right\|_F^2 
    \leq \frac{C_{\mathrm{stat}}\sigma^2_w (d^2_x + d_x d_u) \log(1/\delta)}{\sigma^2_k M_{j} \tau_{k}} 
    + C_{\mathrm{mis,1}}\exp\!\big(- C_{\mathrm{mis,2}}\sigma^2_k\tau_k\big) 
    +C_\textit{het} \epsilon_{\textit{het}}^2,
\end{align}
for all epochs $k \in [k_{\fin}],$ and systems $i \in [m]$. Here $\epsilon_{\mathrm{het}}$ denotes the maximum heterogeneity level across the systems inside the clusters, quantifying the worst-case intra-cluster deviation from the nominal dynamics (see Assumption \ref{assumption: heterogeneity}).\\

\noindent $\bullet$ \textbf{Case 3: Adversarial systems.} When clusters contain both heterogeneous and adversarial systems, we denote by $\mathcal{E}_{\mathrm{est},3}^{(k)}$ the event when the estimation error is bounded as follows:
\begin{align}\label{eq: Eest3}
    \left\| \widehat{\Theta}^{(i)}_k - \Theta^{(i)}_\star  \right\|_F^2 
    &\leq \frac{ C_{\mathrm{stat}}\sigma^2_w (d^2_x + d_x d_u)\log(1/\delta)}{\sigma^2_k \tau_k}\left( \frac{1}{m_{\hat{j}}} + \lambda^2 d_x\right) + C_{\mathrm{het}}(1+\lambda)^2 \epsilon^2_{\mathrm{het}}\notag\\
    &+  C_{\mathrm{mis,1}} \exp(-C_{\mathrm{mis,2}}\sigma^2_k \tau_k),
\end{align}
for all epochs $k \in [k_{\fin}],$ and all systems $i \in [M]$. The latter case will be elaborated in a subsequent section, where we integrate the effects of both heterogeneity and adversarial systems into our system identification guarantees.

The success event in each regime is then defined by
$$
    \mathcal{E}_{\mathrm{success}}^{(j)} := \mathcal{E}_{\mathrm{bound}} \cap \mathcal{E}_{\mathrm{est},j}, 
    \text{ for } j = 1,2,3,
$$
and corresponding failure events $\mathcal{E}_{\mathrm{failure}}^{(j)} := \big(\mathcal{E}_{\mathrm{success}}^{(j)}\big)^c$.\\

\noindent \textbf{Regret decomposition.}  
With these events in place, as in \cite{cassel2020logarithmic}, the expected regret for system $i \in [m]$ decomposes as follows:
\begin{align}\label{eq: regret decomposition appendix}
    \E \left[\mathcal{R}_T^{(i)}\right]
    &= R_1^{(i)} + R_2^{(i)} + R_3^{(i)} - T \mathcal{J}^{(i)}(K_\star^{(i)}), \text{ where} \notag \\
    R_1^{(i)} &= \E \left[ \mathbf{1}\left(\mathcal{E}_{\mathrm{success}}^{(j)}\right) 
        \sum_{k=2}^{k_\fin} J_k^{(i)} \right], \;\ R_2^{(i)} = \E \left[ \mathbf{1}\left(\mathcal{E}_{\mathrm{failure}}^{(j)}\right) 
        \sum_{t=\tau_1+1}^{T} c_t^{(i)} \right], \;\ R_3^{(i)} = \E \left[ \sum_{t=1}^{\tau_1} c_t^{(i)} \right],
\end{align}
with $J_k^{(i)} := \sum_{t=\tau_k}^{\tau_{k+1}-1} c_t^{(i)}$ denoting the cost for epoch $k \in [k_\fin]$. The interpretation for each term is standard:  
\begin{enumerate}
\item $R_1^{(i)}$ quantifies the regret under the success event.  
\item $R_2^{(i)}$ captures the cost contribution under the failure event.  
\item $R_3^{(i)}$ accounts for the one-time exploration cost in the initial epoch.  
\end{enumerate}

A key step for synthesizing the regret bounds is to establish that $\mathbb{P}(\mathcal{E}_{\mathrm{success}}^{(j)}) \geq 1 - T^{-2}$ in each setting $j=1,2,3$ (see Appendix \ref{appendix: probability of success}). This renders the contribution of $R_2^{(i)}$ negligible, and ensures that the dominant terms are $R_1^{(i)}$ and $R_3^{(i)}$, which we bound explicitly in terms of noise variance, heterogeneity, and adversarial contamination.

\section{Auxiliary Results} \label{appendix: auxiliary results}

We now turn to the auxiliary concentration inequalities and matrix norm bounds that play a central role in analyzing our robust multitask adaptive LQR design. In particular, we invoke the matrix Hoeffding's inequality to derive high-probability bounds on the estimation error under sub-Gaussian noise. We further provide derivations of the error decomposition into variance, misclassification bias, and heterogeneity-induced components.  

These results are fundamental for establishing the design rationale and theoretical guarantees of the clustering-based adaptive estimation approach. The proofs are structured around a unified estimator framework and rely on matrix concentration techniques along with the assumptions on the data distribution and honest-majority condition for the adversarial setting.

\begin{lemma}[Matrix Hoeffding \citep{tropp2011user}]
\label{lem:matrix-hoeffding}
Let $\{X_\ell\}_{\ell=1}^m$ be independent, random and symmetric matrices in $\mathbb{R}^{d\times d}$ with
$\mathbb{E}[X_\ell]=0$ and almost-sure bounds $X_\ell^2 \preceq B_\ell^2$ for fixed symmetric matrix $B_\ell \succeq 0$.
Define $\sigma^2 := \left\|\sum_{\ell=1}^m B_\ell^2\right\|$.
Then, for all $t\ge 0$,
\begin{align}\label{Hoeffding ineq - square matrix}
\mathbb{P}\left( \lambda_{\max}\left(\sum_{\ell=1}^m X_\ell\right) \geq t \right) \leq 
d \exp \left( - \frac{t^2}{8 \sigma^2} \right). 
\end{align}

In addition, for general rectangular $\left\{M_\ell\right\}_{\ell=1}^m \subset \mathbb{R}^{d_1 \times d_2}$, we define $X_\ell:=\left[\begin{array}{cc}0 & M_\ell \\ M_\ell^{\top} & 0\end{array}\right]$ and assume each $M_\ell$ satisfies $\mathbb{E}\left[M_\ell\right]=0$. Then for all $t \geq 0$,

\begin{align}\label{Hoeffding ineq - rectangular matrix}
\mathbb{P}\left[\sigma_{\max }\left(\sum_{t=1}^T M_\ell\right) \geq t\right] \leq\left(d_1+d_2\right) \exp \left(-\frac{t^2}{8 \sigma^2}\right).
\end{align}
\end{lemma}

\begin{proof} The proof for this lemma is detailed in \cite{tropp2011user}. The idea is to apply a Laplace-transform method with the fact that $\log\mathbb{E}\exp(\theta X_\ell)\preceq \frac{\theta^2}{2} B_\ell^2$ for $\theta$ small enough when $X_\ell^2\preceq B_\ell^2$ and finish with a matrix Chernoff bound plus a trace trick.
\end{proof}

\begin{lemma}[Bound on $R^{(i)}_1$ from \cite{lee2023nonasymptotic}]\label{lem:R1_cluster}
Let the $\mathcal{E}^{(j)}_\text{success} = \mathcal{E}_{\mathrm{bound}} \cap \mathcal{E}_{\mathrm{est},j} $ for the settings $j = 1,2,3$. Fix a system $i \in [m]$. Then, the $R^{(i)}_1$ as defined in \ref{eq: regret decomposition appendix} is bounded by: 

\begin{align}
R^{(i)}_1 \leq \sum_{k=2}^{k_{\mathrm{fin}}}\Bigg(
&\E\Big[\mathbf{1}(\mathcal{E}_{\mathrm{est},j}^{(k-1)})\,142\,(\tau_k-\tau_{k-1})\,\|P^{(i)}_\star\|^{8}\,\big\|\,[\widehat A^{(i)}_{k-1}\;\widehat B^{(i)}_{k-1}]-[A^{(i)}_\star\;B^{(i)}_\star]\,\big\|_F^2\Big] \notag\\
&+(\tau_k-\tau_{k-1})\,J^{(i)}(K^{(i)}_\star)+4(\tau_k-\tau_{k-1})\,d_u\,\|P^{(i)}_\star\|\,\sigma_k^2\,\Psi^{(i)2}_{B} \notag\\
&+2x_b^2\log T\,\|P^{(i)}_\star\|
\Bigg),
\end{align}
where $\Psi^{(i)}_{B}=\max\{1,\|B^{(i)}_\star\|\}$.
\end{lemma}

\begin{lemma}[Bound on $R^{(i)}_2$ adapted from~\cite{lee2023nonasymptotic}]
\label{lem:R2-bound}
Fix a system $i \in [m]$. Then, in the contribution of the failure event to the regret is bounded as:
\begin{align}
R^{(i)}_2 &\leq 
T^{-1} \big( \|Q\| + 2K_b^2 \big) x_b^2 \log T
+ T^{-1} J^{(i)}(K^{(i)}_0) 
+ 24  \|P^{(i)}_{K^{(i)}_0}\| \Psi^{(i)2}_{B} (d_x + d_u) \sigma^2_w T^{-2} \log(3T) \notag \\
& + 2 T^{-2} \|P^{(i)}_{K^{(i)}_0}\| \, \|\Theta^{(i)}_\star\|_F^2 \, K_b^2 x_b^2 \log T
+ \sum_{k=1}^{k_{\mathrm{fin}}} 2(\tau_k - \tau_{k-1}) d_u \, \sigma_k^2,
\end{align}
where $K_b$ and $x_b$ denote controller and state norm bounds, and $\sigma_k^2$ is the variance of the exploration noise in epoch $k$.
\end{lemma}

\begin{lemma}[Bound on $R^{(i)}_3$ from~\cite{lee2023nonasymptotic}]
\label{lem:R3-bound}
The exploration cost during the first epoch with length $\tau_1$ satisfies:
\begin{align}
R^{(i)}_3 \leq  3 \tau_1 \, \max\{ d_x , d_u \} \, \| P^{(i)}_{K^{(i)}_0} \| \, \Psi^{(i)2}_{B}.
\end{align}
\end{lemma}

We also revisit two additional results from \cite{lee2023nonasymptotic}. The first result in Lemma \ref{lem: noise bound} controls the largest norm of the process and exploration noises, $w^{(i)}_t$ and $x^{(i)}_t$, respectively, for any system $i \in [m]$, with high probability. The second result in Lemma \ref{lem: state rollout bounds} bounds the norm of the state and the norm of the solution of the Lyapunov equation $P_K$, for a sufficiently large horizon length. In particular, the bound for the state norm scales with the largest norm of the process and exploration noise. Later, we see that Lemmas \ref{lem: noise bound} and \ref{lem: state rollout bounds} can be used to demonstrate that if the initial epoch length is sufficiently large, the state and controller norm boundedness requirement in Algorithm \ref{alg: CE with online clustering} are satisfied.

\begin{lemma}(\cite{lee2023nonasymptotic})
    \label{lem: noise bound}
    Let $\delta \in (0,1)$. For any system $i \in [m]$, it holds that 
    \begin{align*}
        \max_{0\leq t \leq T-1} \norm{\bmat{w^{(i)}_t \\ g^{(i)}_t}} \leq 4\sigma_w \sqrt{ (d_x+d_u) \log\frac{T}{\delta}},
    \end{align*}
    with probability at least $1-\delta$.
\end{lemma}

\begin{lemma}(\cite{lee2023nonasymptotic})
    \label{lem: state rollout bounds}
    Consider the discrete LTI system $x_{s+1} = A_{\star} x_s + B_{\star} u_s +w_s$ with initial state $x_0$. Suppose that we play this system with the control action $u_s = K x_s + \sigma_u g_s$ where $K$ is stabilizing and $\sigma_u \leq 1$, for $t$ time steps. Moreover, suppose that
    \begin{itemize}
        \item $\norm{x_1} \leq 16 \norm{P_{K_0}}^{3/2} \Psi_{B^{\star}} \max_{0\leq t \leq T-1} \norm{\bmat{w_t \\ g_t}}$
        \item $\norm{P_K} \leq 2 \norm{P_{K_0}}$
        \item $t \geq \log_{\left(1 - \frac{1}{\norm{P_K}}\right)}\left(\frac{1}{4\norm{P_K}}\right)+1$.
    \end{itemize}  Then for $s=0, \dots, t-1$, it holds that
    \begin{align*}
        \norm{x_s} \leq 40 \norm{P_{K_0}}^2  \Psi_{B^{\star}} \max_{1 \leq t \leq T} \norm{\bmat{w_t \\ g_t}}.
    \end{align*}
    In addition, we have
    \begin{align*}
        \norm{x_t} \leq 16 \norm{P_{K_0}}^{3/2} \Psi_{B^{\star}}\max_{1\leq t \leq T} \norm{\bmat{w_t \\ g_t}}.
    \end{align*}
\end{lemma}

\begin{assume}(Geometric mixing) For any system $i \in [m]$, assume the state evolution -input process $\left\{z_t^{(i)}\right\}_{t \geq 0}$ is a mean-zero stationary $\beta$-mixing process, with stationary covariance $\Sigma^{(i)}_{z_t}$ and $\beta(s)\leq C_\beta \rho^s$, for some $C_\beta \geq 0$ and $\rho \in (0,1)$.
\end{assume}

\section{Multitask System Identification} \label{appendix: multitask sysid}

We now turn our attention to characterizing the error bounds for clustered system identification under three different settings: 
1) without intra-cluster heterogeneity and without adversarial systems; 
2) with intra-cluster heterogeneity but still without adversarial systems; and 3) in the presence of adversarial systems. 
For the first two settings, the aggregation function used in Algorithm~\ref{alg: CE with online clustering} is a simple average, whereas for the third, the adversarial setting, we employ a resilient aggregation scheme as defined in Definition~\ref{def: resilient aggregation}. For completeness, we restate here the key lemmas presented in the main body of the paper.

\subsection{System Identification with Intra-cluster Homogeneity}

We now characterize the system estimation error in epoch $k$ (of length $\tau_k$) for clustered system identification when all systems within a cluster have identical local models.

\begin{lemma}[Estimation error under intra-cluster homogeneity]
\label{lemma: homogeneous-regression-error} Fix a cluster $\mathcal C_j$ with $M_j$ systems.  For each system $i\in \mathcal C_j$, data generation follow \eqref{eq:data generation}. Moreover, suppose Assumption \ref{assumption: initial model estimate} holds. Then, given a small probability of failure $\delta \in (0,1)$, it holds that
\begin{align}
 \left\| \widehat{\Theta}^{(i)} - \Theta^{(i)}_\star \right\|_F^2 \leq \frac{C_{\mathrm{stat}}\sigma^2_w (d^2_x + d_x d_u) \log(1/\delta)}{\sigma^2_k M_{j} \tau_{k}} + C_{\mathrm{mis,1}} \exp\left(-C_{\mathrm{mis,2}} \sigma_k^2 \tau_k\right),   
\end{align}
for any system $i \in \mathcal C_j$, with probability at least $1-\delta-\exp(-C_{\mathrm{mis,2}}\sigma_k^2 \tau_k)$, where $C_{\mathrm{mis,2}} = \mathcal{O}( \Delta_{\min}^2/ \Delta^2_{\max}\sigma^2_w )$ quantifies the decay rate of the misclassification at epoch $k$, and $C_{\mathrm{mis,1}} = \mathcal{O}(\Delta_{\max}^2)$ corresponds is the misclassification constant.
\end{lemma}
\begin{proof}
For a given epoch $k$ of length $\tau_k$, each system $i \in [M]$ evolves according to the dynamics in~\eqref{eq:data generation}. Each system uses its collected dataset $\mathcal{D}^{(i)} = \{X^{(i)}, Z^{(i)}\}$, consisting of state–input and next-state trajectories, to estimate its model and identify its cluster membership. Suppose that system $i$ is assigned to cluster $\hat{j}$, where $\mathcal C_{\hat{j}}$ denotes the set of indices of systems belonging to cluster $\hat{j} \in [N_c]$. Each system computes its gradient using its local data and the common model of the cluster to which it is assigned. The gradients are then transmitted to the server, which performs the following aggregation step:
\begin{align*}
    \widehat{\Theta}^{(i)}_{n+1} = \widehat{\Theta}^{(i)}_{n} +  \frac{\eta}{|\mathcal C_{\hat{j}}|} \sum_{\ell \in \mathcal \mathcal C_{\hat{j}}} G_{\ell}(\widehat{\Theta}^{(i)}_{n}),
\end{align*}
where $G_{\ell}(\widehat{\Theta}^{(i)}) = (X^{(\ell)} - \widehat{\Theta}^{(i)} Z^{(\ell)}) Z^{(\ell)\top} (Z^{(\ell)} Z^{(\ell)\top})^{-1}$, as specified in Algorithm~\ref{algorithm: clustered sysid}.  We can further decompose the average over the set of systems assigned to cluster $\hat{j}$ into two parts:  the average over systems that are correctly identified, i.e., those belonging to $\mathcal C_{j} \cap \mathcal C_{\hat{j}}$, 
and the average over systems that are misclassified to cluster $\hat{j}$, i.e., those belonging to $\mathcal C_{j}^{c} \cap \mathcal C_{\hat{j}}$, 
where $\mathcal C_{j}^{c}$ denotes the complement of $\mathcal C_{j}$. 
That is, we obtain
\begin{align*}
    \widehat{\Theta}^{(i)}_{n+1} = \widehat{\Theta}^{(i)}_{n} +  \frac{\eta}{|\mathcal C_{\hat{j}}|} \sum_{\ell \in \mathcal C_j \cap \mathcal C_{\hat{j}} } G_{\ell}(\widehat{\Theta}^{(i)}_{n}) +  \frac{\eta}{|\mathcal C_{\hat{j}}|} \sum_{\ell \in \mathcal C^c_j \cap \mathcal C_{\hat{j}} } G_{\ell}(\widehat{\Theta}^{(i)}_{n}),
\end{align*}
where we can write 
\begin{align*}
   \frac{1}{|\mathcal C_{\hat{j}}|} \sum_{\ell\in \mathcal C_{j} \cap \mathcal C_{\hat{j}}} G_l(\widehat{\Theta}^{(i)}_n)&= \frac{1}{|\mathcal C_{\hat{j}}|} \sum_{\ell\in \mathcal C_{j} \cap \mathcal C_{\hat{j}}} (X^{(\ell)} - \widehat{\Theta}^{(i)}Z^{(\ell)})Z^{(\ell)\top} (Z^{(\ell)}Z^{(\ell)\top})^{-1}\\
    &=\frac{1}{|\mathcal C_{\hat{j}}|} \sum_{\ell\in \mathcal C_{j} \cap \mathcal C_{\hat{j}}} (\Theta^{(i)}_\star Z^{(\ell)} + W^{(\ell)} - \widehat{\Theta}^{(i)}Z^{(\ell)})Z^{(\ell)\top} (Z^{(\ell)}Z^{(\ell)\top})^{-1}\\
    &=\frac{1}{|\mathcal C_{\hat{j}}|} \sum_{\ell\in \mathcal C_{j} \cap \mathcal C_{\hat{j}}} \left(\Theta^{(i)}_\star Z^{(\ell)} + W^{(\ell)} - \widehat{\Theta}^{(i)}Z^{(\ell)}\right)Z^{(\ell)\top} (Z^{(\ell)}Z^{(\ell)\top})^{-1}\\
    &= \Theta^{(i)}_\star - \widehat{\Theta}^{(i)}_n + \frac{1}{|\mathcal C_{\hat{j}}|} \sum_{\ell\in \mathcal C_{j} \cap \mathcal C_{\hat{j}}} W^{(\ell)}Z^{(\ell)\top}(Z^{(\ell)}Z^{(\ell)\top})^{-1},
\end{align*}
and the following expression for the average over the misclassified systems:
\begin{align*}
    \frac{1}{|\mathcal C_{\hat{j}}|} \sum_{\ell \in \mathcal C^c_j \cap \mathcal C_{\hat{j}} } G_{\ell}(\widehat{\Theta}^{(i)}_{n}) = \frac{1}{|\mathcal C_{\hat{j}}|} \sum_{\ell\in \mathcal C^c_{j} \cap \mathcal C_{\hat{j}}} (\Theta_j - \widehat{\Theta}^{(i)}) + \frac{1}{|\mathcal C_{\hat{j}}|} \sum_{\ell\in \mathcal C^c_{j} \cap \mathcal C_{\hat{j}}}  W^{(\ell)}Z^{(\ell)\top} (Z^{(\ell)}Z^{(\ell)\top})^{-1}
\end{align*}
where $\Theta_j$ denotes the common model that system $q$ would use if it were correctly classified to its true cluster, which is different from $\hat{j}$. Therefore, we obtain

\begin{align}\label{eq: error per iteration}
    \widehat{\Theta}^{(i)}_{n+1} = \widehat{\Theta}^{(i)}_{n} + \eta\left(\Theta^{(i)}_\star - \widehat{\Theta}^{(i)}_n\right) + \underbrace{\frac{\eta}{|\mathcal C_{\hat{j}}|} \sum_{\ell\in \mathcal C_{\hat{j}}} W^{(\ell)}Z^{(\ell)\top}(Z^{(\ell)}Z^{(\ell)\top})^{-1}}_{\text{statistical error}} + \underbrace{\frac{\eta}{|\mathcal C_{\hat{j}}|} \sum_{\ell\in \mathcal C^c_{j} \cap \mathcal C_{\hat{j}}} (\Theta_j - \widehat{\Theta}^{(i)})}_{\text{misclassification error}},
\end{align}
where we subtract $\Theta^{(i)}_{\star}$ from both sides to obtain 
\begin{align*}
   \|\widehat{\Theta}^{(i)}_{n+1} - \Theta^{(i)}_{\star}\|_F  &\leq  (1-\eta)\|\widehat{\Theta}^{(i)}_{n} - \Theta^{(i)}_{\star}\|_F + \left\| \frac{\eta}{|\mathcal C_{\hat{j}}|} \sum_{\ell\in \mathcal C_{\hat{j}}} W^{(\ell)}Z^{(\ell)\top}(Z^{(\ell)}Z^{(\ell)\top})^{-1}\right\|_F \hspace{-0.2cm}\\
   &+ \left\|\frac{\eta}{|\mathcal C_{\hat{j}}|} \sum_{\ell\in \mathcal C^c_{j} \cap \mathcal C_{\hat{j}}} (\Theta_j - \widehat{\Theta}_\star^{(i)})\right\|_F\hspace{-0.3cm},
\end{align*}

\noindent $\bullet$ \textbf{Statistical Error:} For each system $\ell\in \mathcal C_{\hat{j}}$, we have trajectory data of length $\tau_k=\tau_{k,1}+\tau_{k,2}$ and the block split
$\mathcal{I}_1$ (size $\tau_{k,1}$) and $\mathcal{I}_2$ (size $\tau_{k,2}$), with $\tau_{k,1},\tau_{k,2}\geq 1$.
We recall that $z^{(\ell)}_t=\begin{bmatrix}x^{(\ell)}_t \\ u^{(\ell)}_t\end{bmatrix}\in\mathbb{R}^{d}$ denotes the state-input vector,
and $w^{(\ell)}_t\in\mathbb{R}^{d_x}$ is the process noise, independent across $t$ and $\ell$, with zero mean. Define
$$
Z^{(\ell)}_{\mathcal{I}_1} := \left[z^{(\ell)}_t \right]_{t\in\mathcal{I}_1}\in\mathbb{R}^{d\times \tau_{k,1}}, \;\ Z^{(\ell)}_{\mathcal{I}_2} := \left[z^{(\ell)}_t \right]_{t\in\mathcal{I}_2}\in\mathbb{R}^{d\times \tau_{k,2}}, \;\
P^{(\ell)} := \left(Z^{(\ell)}_{\mathcal{I}_1} Z^{(\ell)\top}_{\mathcal{I}_1}\right)^{-1}\in\mathbb{R}^{d\times d}.
$$

We also define the (decoupled) cross term on the second block
$$
M^{(\ell)} := \sum_{t\in\mathcal{I}_2} w^{(\ell)}_t z^{(\ell)\top}_t P^{(\ell)} \in \mathbb{R}^{d_x \times d},\;\
\bar{M} \;:=\; \frac{1}{|\mathcal C_{\hat{j}}|}\sum_{\ell \in \mathcal C_{\hat{j}}} M^{(\ell)}.
$$

\noindent \textbf{Sample-split independence:} By construction (i.e., sample-split in two independent blocks), $P^{(\ell)}$ is independent of $\{(w^{(\ell)}_t,z^{(\ell)}_t):t\in\mathcal{I}_2\}$, and thus
$\mathbb{E}[M^{(\ell)}\mid P^{(\ell)}]=0$. We can see this by fixing a system $\ell$ and a time $t\in\mathcal{I}_2$, and by considering the $\sigma$–field
$\mathcal{F}_{t-1}:=\sigma \left(P^{(\ell)},\{z^{(\ell)}_s:s\le t\}\right)$.

Therefore, the sample–split construction, $P^{(\ell)}$ depends only on the block $\mathcal{I}_1$ and is then independent of the pair
$\left(w^{(\ell)}_t,z^{(\ell)}_t\right)$ from block $\mathcal{I}_2$. In the process–noise model,
$w^{(\ell)}_t$ is zero–mean and it only affects $x^{(\ell)}_{t+1}$ (hence $z^{(\ell)}_{t+1}$), but is independent of the present regressor
$z^{(\ell)}_t$ and of $\mathcal{F}_{t-1}$. Consequently,
$$
\mathbb{E}\left[w^{(\ell)}_t\,z^{(\ell)\top}_t P^{(\ell)} \mid \mathcal{F}_{t-1}\right]
= \left(\mathbb{E}[w^{(\ell)}_t \mid \mathcal{F}_{t-1}]\right)\, z^{(\ell)\top}_t P^{(\ell)} = 0,
$$
where we used that $z^{(\ell)\top}_t P^{(\ell)}$ is $\mathcal{F}_{t-1}$–measurable and $\mathbb{E}[w^{(\ell)}_t \mid \mathcal{F}_{t-1}]=0$.
Taking expectations again and summing over $t\in\mathcal{I}_2$ yields
$$
\mathbb{E}\left[M^{(\ell)} \mid P^{(\ell)}\right]
=\sum_{t\in\mathcal{I}_2}
\mathbb{E}\left[w^{(\ell)}_t z^{(\ell)\top}_t P^{(\ell)} \mid  P^{(\ell)}\right]=0.
$$

Note that conditioning must be on $P^{(\ell)}$ (or on the past $\mathcal{F}_{t-1}$). If one conditioned on the entire future block
$Z^{(\ell)}_{\mathcal{I}_2}$, then $w^{(\ell)}_t$ would correlate with $z^{(\ell)}_{t+1}$ and the conditional mean need not be zero.

We proceed by writing
$$
\bar{M}
=\frac{1}{|\mathcal C_{\hat{j}}|}\sum_{\ell \in \mathcal C_{\hat{j}}} \left(\sum_{t\in\mathcal{I}_2} w^{(\ell)}_t\, z^{(\ell)\top}_t\, P^{(\ell)}\right).
$$
We fix a system $\ell$, and by the independence from the sample split we have that $\mathbb{E}[ M^{(\ell)}\mid P^{(\ell)}]=0$. In addition, we for each $t\in\mathcal{I}_2$, we obtain
$$
\|w^{(\ell)}_t z^{(\ell)\top}_t P^{(\ell)}\|
\leq
\|w^{(\ell)}_t\|_2  \|z^{(\ell)}_t\|_2 \|P^{(\ell)}\|,
$$

As we assume $w^{(\ell)}_t$ is mean-zero sub-Gaussian with variance $\sigma_w^2$, and thus $z^{(\ell)}_t$ is mean-zero sub-Gaussian with covariance
$\Sigma^{(\ell)}_{z_t} := \mathbb E[z^{(\ell)}_t z^{(\ell)\top}_t]$ in the sense that
$\| \langle v, z^{(\ell)}_t\rangle \|_{\psi_2} \le C \sqrt{v^\top \Sigma^{(\ell)}_{z_t} v}$ for all $v$ and some constant $C > 0$.
Let $\kappa^{(\ell)} := \lambda_{\max}\left(\Sigma^{(\ell)}_{z_t}\right)$. In particular, by \cite[Lemma 1]{wang2023fedsysid}, we have 
$$
\Sigma^{(\ell)}_{z_t} \triangleq\left[\begin{array}{cc}
\sigma_k^2 G^{(\ell)}_t\left(G^{(\ell)}_t\right)^{\top}+\sigma_w^2 F^{(\ell)}_t\left(F^{(\ell)}_t\right)^{\top} & 0 \\
0 & \sigma_k^2 I_{d_u}
\end{array}\right],
$$
with $G_t \triangleq\left[\begin{array}{llll}
A^{(\ell)t-1} B^{(\ell)} & A^{(\ell)t-2} B^{(\ell)} & \cdots & B^{(\ell)}
\end{array}\right] \text { and } F^{(\ell)}_t \triangleq\left[\begin{array}{llll}
A^{(\ell)t-1} & A^{(\ell)t-2} & \cdots & I_{d_x}
\end{array}\right]$, for any $t\geq 1$. Then, by standard inequalities for sub-Gaussian vectors, on a high-probability event, with probability at least $1-\delta$,
$$
\max_{t\in \mathcal I_2} \|w^{(\ell)}_t\|_2 \leq C_w \sigma_w \sqrt{\log \frac{\tau_{k,2}}{\delta}} := B_w,\;\
\max_{t\in \mathcal I_2} \|z^{(\ell)}_t\|_2 \leq C_z \sqrt{\kappa^{(\ell)}} 
\left(\sqrt{d^\prime} + \sqrt{\log \frac{\tau_{k,2}}{\delta}}\right) := B_z.
$$
for some constants $C_w$ and $C_z$. Hence, we obtain
$$
\|w^{(\ell)}_t z^{(\ell)\top}_t P^{(\ell)}\|
\leq
B_w  B_z \|P^{(\ell)}\|,
$$

We can now apply matrix Hoeffding inequality (Lemma~\ref{lem:matrix-hoeffding}) to
$ M^{(\ell)}$:
$$
\mathcal{X}_\ell= \begin{bmatrix}
0 &  M^{(\ell)}\\
 M^{(\ell)\top} & 0
\end{bmatrix},\;\
\mathbb{E}[\mathcal{X}_\ell\mid P^{(\ell)}]=0,
\;\
\mathcal{X}_\ell^2 \preceq B^2_w B^2_z \|P^{(\ell)}\|^2 I_{2d_x + d_u},
$$

Therefore, we have that
$$
\mathbb{P}\left( \left\|\sum_{\ell \in C_{\hat{j}}} M^{(\ell)}\right\| \geq t \right)
\leq 
(2d_x + d_u) \exp\left( - \frac{t^2}{8\sigma^2} \right),
$$
where $\sigma^2 = B^2_w B^2_z\sum_{\ell \in C_{\hat{j}}} \|P^{(\ell)}\|^2$. By setting $t = \sqrt{\tau_{k}} \sigma_k B_w B_z \sqrt{\sum_{\ell \in C_{\hat{j}}} \|P^{(\ell)}\|^2} \sqrt{8 \log\left(\frac{2d_x + d_u}{\delta}\right)}$, we obtain the following expression:
$$
\left \| \bar{M}\right\| \leq  \frac{\sqrt{\tau_{k}} \sigma_k B_w B_z}{|C_{\hat{j}}|} \sqrt{\sum_{\ell \in C_{\hat{j}}} \|P^{(\ell)}\|^2} \sqrt{8 \log\left(\frac{2d_x + d_u}{\delta}\right)},
$$
with probability at least $1-\delta$.\\

\noindent \textbf{Controlling $\|P^{(\ell)}\|$:} We now proceed to prove that $\|P^{(\ell)}\| \leq \frac{C_P}{\sigma^2_k \tau_k}$, for some constant $C_P$, with probability $1 - \delta_P$. To prove this bound, we first assume that $\{z^{(\ell)}_t\}_{t\in\mathcal I_1}$ is a strictly stationary, mean-zero Gaussian process in $\mathbb{R}^d$ with
covariance $\Sigma^{(\ell)}_{z} := \E\left[z^{(\ell)}_t z^{(\ell)\top}_t\right]\succeq \sigma_k^2 I_d$
and geometric $\beta$-mixing, i.e., $\beta(s)\le C_\beta \rho^s$ for some $C_\beta>0$ and  $\rho\in(0,1)$ (see \citep{yu1994rates}).

Let $\Sigma^{1/2}_t:=(\Sigma^{(\ell)}_{z_t})^{1/2}$. As we know that $\{z_t\}_t$ are Gaussian with covariance $\Sigma^{(\ell)}_{z_t}$,
we can write $z_t=\Sigma^{1/2} y_t$ where $\{y_t\}_t$ is a stationary Gaussian process in $\mathbb{R}^d$ with the same $\beta$-mixing rate. Hence, we can write
$$
S-\Sigma^{(\ell)}_{z_t}
= \Sigma^{1/2}_t\left(\frac{1}{\tau_{k,1}}\sum_{t} y_t y_t^\top - I_d\right)\Sigma^{1/2}_t,
\text{ and thus }
\|S-\Sigma^{(\ell)}_{z_t}\| \leq \|\Sigma^{(\ell)}_{z_t}\| \left\|\frac{1}{\tau_{k,1}}\sum_{t} y_t y_t^\top - I_d\right\|.
$$

As $\{z_{t}\}_t$ are dependent over time $t$, we use a blocking technique to construct ``independent blocks" of data in $\mathcal I_1$. Therefore, let us partition $\mathcal I_1$ into $q$ kept blocks of length $m_1$ separated by gaps of length $m_2$. In addition, we denote
$m=m_1+m_2$, $q=\lfloor \tau_{k,1}/m\rfloor$. For the kept blocks, we define the block averages
$\overline{Y}_j := \frac{1}{m_1}\sum_{t\in \mathcal B_j} (y_t y_t^\top - I_d)$. Therefore, there exist i.i.d. copies
$\{\overline{Y}'_j\}_{j=1}^q$ with the same marginals such that
$\mathbb{P}\{\exists j:\overline{Y}_j\neq \overline{Y}'_j\}\leq q \beta(m_2)\le q C_\beta \rho^{m_2}.$
This implies that, with probability at least $1-q C_\beta \rho^{m_2}$, the kept blocks behave as independent.

Conditional on the coupling event, we apply an effective-rank covariance deviation bound to
$\frac{1}{q}\sum_{j=1}^q \overline{Y}'_j$, where each $\overline{Y}'_j$ is an average of $m_1$ i.i.d. variables, where its sub-exponential norm is uniformly bounded, and its second moment has effective rank at most $r_{\mathrm{eff}}^{(\ell)} = \operatorname{tr}(\Sigma_t)/\|\Sigma_t\|$.
A standard Gaussian covariance concentration (e.g., matrix Bernstein \citep{vershynin2018high}) yields, for all $s>0$,
$$
\mathbb{P}\left\{\left\|\frac{1}{q}\sum_{j=1}^q \overline{Y}'_j\right\|\ \geq
C_Y\left(\sqrt{\frac{r_{\mathrm{eff}}^{(\ell)} + s}{qm_1}}\ +\ \frac{r_{\mathrm{eff}}^{(\ell)} + u}{qm_1}\right)\right\}
\leq 2 e^{-s}.
$$

As the contribution of discarded intervals of samples dilutes the kept sample size by at most a factor $\frac{m_1}{m_1+m_2}$,
we conclude that
$$
\left\|\frac{1}{\tau_{k,1}}\sum_{t} y_t y_t^\top - I_d\right\|
\leq C_Y\left(\sqrt{\frac{r_{\mathrm{eff}}^{(\ell)} + \log(2/\delta)}{\tau_{k,1}}} + \frac{r_{\mathrm{eff}}^{(\ell)} + \log(2/\delta)}{\tau_{k,1}}\right)$$
with probability at least $1-\delta - q C_\beta \rho^{m_2}$, for some constant $C_Y$.

By setting $m_2=\left\lceil \frac{\log(2 q C_\beta/\delta)}{|\log\rho|}\right\rceil$ we have that
$q C_\beta \rho^{m_2} \leq \delta/2$, and by taking $m_1=m_2$ so that $qm_1\asymp \tau_{k,1}$. Then the bound above holds with probability at least $1-\delta$ and becomes 
$$
\|S-\Sigma^{(\ell)}_{z_t}\| \leq  C_Y\|\Sigma^{(\ell)}_{z_t}\| \left(\sqrt{\frac{r_{\mathrm{eff}}^{(\ell)} + \log(2/\delta)}{\tau_{k,1}}} + \frac{r_{\mathrm{eff}}^{(\ell)} + \log(2/\delta)}{\tau_{k,1}}\right).
$$

Therefore, by setting $\tau_{k,1}\geq C_Y(r_{\mathrm{eff}}^{(\ell)}+\log(2/\delta))$ with a large enough $C_Y$, we have 
$$
\|P^{(\ell)}\| \leq \frac{2}{\tau_{k,1} \lambda_{\min}(\Sigma^{(\ell)}_{z})} \leq \frac{2}{\tau_{k,1} \sigma^2_k},
$$
as we know that $\Sigma^{(\ell)}_{z_t} \succeq \sigma^2_k I_d$, by persistency of excitation, and thus $\lambda_{\min}(\Sigma^{(\ell)}_{z}) \geq \sigma^2_k$.

Therefore, by using the above bound for $\|P^{(\ell)}\|$, we guarantee that the statistical error term in \eqref{eq: error per iteration} is upper bounded as follows:

\begin{align}\label{statistical error}
\left\| \frac{1}{|\mathcal C_{\hat{j}}|} \sum_{\ell\in \mathcal C_{\hat{j}}} W^{(\ell)}Z^{(\ell)\top}(Z^{(\ell)}Z^{(\ell)\top})^{-1}\right\|_F &\leq  \frac{ \sqrt{d_x}\sqrt{\tau_{k,1}} \sigma_k B_w B_z}{|C_{\hat{j}}|} \sqrt{|C_{\hat{j}}|} \frac{C_P}{\sigma^2_k \tau_k} \sqrt{8 \log\left(\frac{2d_x + d_u}{\delta}\right)}\notag\\
&=  \frac{ \sqrt{d_x} B_w B_z C_P}{\sigma_k} \sqrt{\frac{8}{|\mathcal C_{\hat{j}}| \tau_{k,1}} \log\left(\frac{2d_x + d_u}{\delta}\right)}.
\end{align}
where the extra $\sqrt{d_x}$ term is due to upper bounding the Frobenius norm with the spectral norm.\\

\noindent $\bullet$ \textbf{Misclassification Error:} Now let us control the misclassification error term in \eqref{eq: error per iteration}. For this purpose, we define the misclassification event for agent $\ell \in \mathcal C_{\hat{j}}$ as follows:
$$
\mathcal{E}_{\text{mis}}^{(\ell)} := \left\{ \exists j' \neq j : \left\| X^{(\ell)} - \widehat{\Theta}_j Z^{(\ell)} \right\|_F^2 > \left\| X^{(\ell)} - \widehat{\Theta}_{j'} Z^{(\ell)} \right\|_F^2 \right\},
$$
where we recall that $X^{(\ell)}$ and $Z^{(\ell)}$ denote the data matrices corresponding to system $\ell$, and $\widehat{\Theta}_{j'}$ is the estimated model for cluster $j' \neq j$. Substituting $X^{(\ell)} = \Theta_j Z^{(\ell)} + W^{(\ell)},$ where $\Theta_j = \Theta^{(\ell)}_\star$, we obtain
$$
\left\| (\Theta_j - \widehat{\Theta}_j) Z^{(\ell)} + W^{(\ell)} \right\|_F^2 > \left\| (\Theta_j - \widehat{\Theta}_{j'}) Z^{(\ell)} + W^{(\ell)} \right\|_F^2.
$$

Rearranging terms, this event is equivalent to
$$
D := \left\| \Delta_{j'} Z^{(\ell)} + W^{(\ell)} \right\|_F^2 - \left\| \Delta_j Z^{(\ell)} + W^{(\ell)} \right\|_F^2 < 0,
$$
where we define $\Delta_{j'} := \Theta_j - \widehat{\Theta}_{j'}$ and $\Delta_j := \Theta_j - \widehat{\Theta}_j$ as the estimation residuals for the incorrect and correct cluster models, respectively. We proceed to analyze this term by defining 
$$
D_1 := \left\| \Delta_{j'} Z^{(\ell)} \right\|_F^2 - \left\| \Delta_j Z^{(\ell)} \right\|_F^2, \text{ and } D_2 := 2 \mathrm{tr}\left( W^{(\ell)\top} (\Delta_{j'} - \Delta_j) Z^{(\ell)} \right),
$$
where $D = D_1 + D_2$. Then the probability of misclassification probability becomes
$$
\mathbb{P}(\mathcal{E}_{\text{mis}}^{(\ell)}) = \mathbb{P}(D < 0) = \mathbb{P}(D_1 < -D_2).
$$

As $D_1$ is a linear with respect to the assumed Gaussian noise $W^{(\ell)}$ matrix, it is sub-Gaussian with variance bounded as follows:
$$
\operatorname{Var}(D_2) \leq C_{D_2}\sigma^2_w  \mathbb{E}\left\| (\Delta_{j'} - \Delta_j) Z^{(\ell)} \right\|_F^2.
$$
for some sufficiently large constant $C_{D_2}$. Moreover, by bounding $\| \Delta_{j'} - \Delta_j \|_F \leq 2\Delta_{\max}$ with the maximum cluster separation, we have that 
$$
\operatorname{Var}(D_2) \leq 2C_{D_2}\sigma^2_w \Delta^2_{\max}  \mathbb{E}\left\| Z^{(\ell)} \right\|_F^2.
$$

In addition, we have that $\| Z^{(\ell)} \|_F^2 \leq C_{Z,\text{spec}} \tau_k $, where $C_{Z,\text{spec}}$ is a uniform upper bound for the spectral norm of $Z^{(\ell)}$. Therefore, we obtain

$$
\operatorname{Var}(D_2) \leq 2C_{D_2}\Delta^2_{\max} C_{Z,\text{spec}} \sigma^2_w\tau_k,
$$

To lower bound $D_1$, observe that by persistency of excitation, $\lambda_{\min}(\mathbb{E}[z^{(i)}_t z^{(i)\top}_t])\geq \sigma_k^2$, and Assumption \ref{assumption: initial model estimate}, we obtain
$$
D_1 \geq \tau_k \lambda_{\min}\left( \mathbb{E}[z_t z_t^\top] \right) \left( \| \Delta_{j'} \|_F^2 - \| \Delta_j \|_F^2 \right)
\geq \sigma_k^2 \tau_k \left( \Delta_{\min}^2 - \delta_{\text{stat}}^2 \right).
$$
where $\delta_{\text{stat}}$ denotes the statistical error from \eqref{statistical error}. Therefore, by setting $\tau_{k,1}$ such that
$$
\frac{8 d_x B_w B_z C_P}{\sigma^2_k |C_{\hat{j}}| \tau_{k,1}} \log\left(\frac{2d_x + d_u}{\delta}\right) \leq \frac{\Delta^2_{\min}}{2} \rightarrow   \tau_{k,1} \geq \frac{8 d_x B_w B_z C_P \Delta^2_{\min}}{2\sigma^2_k M_{\hat{j}}} \log\left(\frac{2d_x + d_u}{\delta}\right),
$$
which is guaranteed by making the initial epoch length sufficiently large. Therefore, we obtain
$$
D_1 \geq \frac{\sigma^2_k \tau_k \Delta^2_{\min}}{2} \geq \frac{\sigma_k \tau_k \Delta_{\min}}{2},
$$
where the second inequality follows from assuming that the clusters are sufficiently well separated such that $\Delta_{\min} \geq \frac{1}{\sigma_k}$. By substituting into a Gaussian tail bound yields the following misclassification probability:
$$
\mathbb{P}(\mathcal{E}_{\text{mis}}^{(\ell)}) \leq \exp\left( - \frac{D_1^2}{2 D_2} \right)
\leq \exp\left( - \tau_k \frac{ \sigma_k^2  \Delta_{\min}^2 }{ 8 C_{D_2} C_{Z,\text{spec}} \Delta^2_{\max}  \sigma^2_w } \right) := \exp\left( - C_2 \sigma_k^2 \tau_k  \right).
$$ 
with $C_2 = \frac{\Delta^2_{\min}}{ 8 C_{D_2} C_{Z,\text{spec}} \Delta^2_{\max}  \sigma^2_w }$.
Therefore, the misclassification rate is bounded by
$$
\frac{|\mathcal C^{c}_{j} \cap \mathcal C_{\hat{j}}|}{|\mathcal C_{\hat{j}}|} \leq \exp(-C_2 \sigma_k^2  \tau_k),
$$
which implies that
\begin{align}\label{misclassification error}
    \left\|\frac{1}{|\mathcal C_{\hat{j}}|} \sum_{\ell\in \mathcal C^c_{j} \cap \mathcal C_{\hat{j}}} (\Theta_j - \widehat{\Theta}_\star^{(i)})\right\|_F \leq \Delta_{\max} \exp\left(-C_2 \sigma_k^2 
    \frac{\tau_k}{2}\right),
\end{align}

Combining the high-probability bound on the statistical error \eqref{statistical error} with the misclassification error \eqref{misclassification error}, we obtain the following error bound at iteration $n$:

\begin{align*}
   \|\widehat{\Theta}^{(i)}_{n+1} - \Theta^{(i)}_{\star}\|_F  &\leq  (1-\eta)\|\widehat{\Theta}^{(i)}_{n} - \Theta^{(i)}_{\star}\|_F +   \frac{\eta \sqrt{d_x} B_w B_z C_P}{\sigma_k} \sqrt{\frac{8}{|\mathcal C_{\hat{j}}| \tau_{k,1}} \log\left(\frac{2d_x + d_u}{\delta}\right)}\\
   &+\eta \Delta_{\max} \exp\left(-C_2 \frac{\tau_k}{2}\right)\\
   &=   (1-\eta)\|\widehat{\Theta}^{(i)}_{n} - \Theta^{(i)}_{\star}\|_F +   \frac{\eta  C_{\mathrm{stat}}\sigma_w \sqrt{(d^2_x + d_x d_u) \log(1/\delta)}}{\sigma_k\sqrt{M_{\hat{j}} \tau_{k,1}}}\\
   &+\eta C_{\mathrm{mis,1}} \exp\left(-C_{\mathrm{mis,2}}\sigma_k^2 \tau_k\right),
\end{align*}
where $C_{\mathrm{stat}}$ depends on $C_w, \log(\tau_{k,2}), C_z, \kappa^{(\ell)}, C_P,$ and $\log(d)$. In addition, for clarity in our bounds, we use $C_{\mathrm{mis,1}} = \Delta_{\max}$ and $C_{\mathrm{mis,2}} = \frac{C_2}{2}$.

To conclude the proof we unroll the above expression over $N$ iterations and write
\begin{align*}
   \|\widehat{\Theta}^{(i)}_{N} - \Theta^{(i)}_{\star}\|_F  \leq   \rho^N\|\widehat{\Theta}^{(i)}_{0} - \Theta^{(i)}_{\star}\|_F +   \frac{C_{\mathrm{stat}}\sigma_w \sqrt{(d^2_x + d_x d_u)\log(1/\delta)}}{\sigma_k\sqrt{M_{\hat{j}} \tau_{k,1}}} + C_{\mathrm{mis,1}} \exp\left(-C_{\mathrm{mis,2}} \sigma_k^2 \tau_k\right),
\end{align*}
with $\rho = 1-\eta$. The proof is complete by setting $
N \geq \log\left(C_\alpha \Delta_{\text{min}}\tau^2_1\right)/\log(1/\rho),
$
and noting that for a sufficiently large initial epoch length, $M_{\hat{j}} \approx M_j$, which guarantees that the systems are correctly classified and the error bound decays with the true total number of systems inside the cluster. Finally, we also note that $\tau_{k,1} = \mathcal{O}(\tau_k)$. In addition, we omit the contraction term of order $\mathcal{O}(1/\tau_k^2)$ by setting the total number of iterations $N$ as before, since this term becomes negligible in the subsequent regret analysis.
\end{proof}

\subsection{System Identification with Intra-cluster Heterogeneity}

We now turn our attention to the setting in which the systems within each cluster are similar but not identical, exhibiting bounded heterogeneity characterized by $\epsilon_{\mathrm{het}}$.

\begin{proposition}[Estimation error under intra-cluster heterogeneity]
\label{prop: error bound heterogeneity} Fix a cluster $\mathcal C_j$ with $M_j$ systems.  For each system $i\in \mathcal C_j$, data generation follow \eqref{eq:data generation}. Suppose Assumption \ref{assumption: initial model estimate} holds and that the intra-cluster heterogeneity is bounded by $\epsilon_{\mathrm{het}}$.  Then, given a small probability of failure $\delta \in (0,1)$, it holds that
\begin{align}\label{eq: error with heterogeneity}
 \left\| \widehat{\Theta}^{(i)} - \Theta^{(i)}_\star \right\|_F^2 \leq \frac{C_{\mathrm{stat}}\sigma^2_w (d^2_x + d_x d_u)}{\sigma^2_k M_{j} \tau_{k}} +  C_{\mathrm{het}} \epsilon^2_{\mathrm{het}} + C_{\mathrm{mis,1}} \exp\left(-C_{\mathrm{mis,2}} \sigma_k^2 \tau_k\right), 
\end{align}
for any system $i \in \mathcal C_j$, with probability at least $1-\delta-\exp(-C_{\mathrm{mis,2}}\sigma_k^2\tau_k)$.
\end{proposition}

\begin{proof}
The proof for this lemma follows from using the of Lemma \ref{lemma: homogeneous-regression-error} and writing 
\begin{align*}
    \widehat{\Theta}^{(i)}_{n+1} = \widehat{\Theta}^{(i)}_{n} +  \frac{\eta}{|\mathcal C_{\hat{j}}|} \sum_{\ell \in \mathcal C_j \cap \mathcal C_{\hat{j}} } G_{\ell}(\widehat{\Theta}^{(i)}_{n}) +  \frac{\eta}{|\mathcal C_{\hat{j}}|} \sum_{\ell \in \mathcal C^c_j \cap \mathcal C_{\hat{j}} } G_{\ell}(\widehat{\Theta}^{(i)}_{n}),
\end{align*}
where we can write for the statistical error term
\begin{align*}
   \frac{1}{|\mathcal C_{\hat{j}}|} \sum_{\ell\in \mathcal C_{j} \cap \mathcal C_{\hat{j}}} &G_l(\widehat{\Theta}^{(i)}_n)= \frac{1}{|\mathcal C_{\hat{j}}|} \sum_{\ell\in \mathcal C_{j} \cap \mathcal C_{\hat{j}}} (X^{(\ell)} - \widehat{\Theta}^{(i)}Z^{(\ell)})Z^{(\ell)\top} (Z^{(\ell)}Z^{(\ell)\top})^{-1}\\
    &=\frac{1}{|\mathcal C_{\hat{j}}|} \sum_{\ell\in \mathcal C_{j} \cap \mathcal C_{\hat{j}}} \left(\left(\textcolor{black}{\Theta^{(i)}_\star} + \Theta^{(\ell)}_\star - \textcolor{black}{\Theta^{(i)}_\star}\right) Z^{(\ell)} + W^{(\ell)} - \widehat{\Theta}^{(i)}Z^{(\ell)}\right)Z^{(\ell)\top} (Z^{(\ell)}Z^{(\ell)\top})^{-1}\\
    &=\frac{1}{|\mathcal C_{\hat{j}}|} \sum_{\ell\in \mathcal C_{j} \cap \mathcal C_{\hat{j}}} \left(\Theta^{(i)}_\star Z^{(\ell)} + W^{(\ell)} - \widehat{\Theta}^{(i)}Z^{(\ell)}\right)Z^{(\ell)\top} (Z^{(\ell)}Z^{(\ell)\top})^{-1}\\
    &+ \underbrace{\frac{\eta}{|\mathcal C_{\hat{j}}|} \sum_{\ell\in \mathcal C_{j} \cap \mathcal C_{\hat{j}}} \left(\Theta^{(\ell)}_\star - \Theta^{(i)}_\star\right)}_{\text{System heterogeneity}},
\end{align*}
where we note the presence of the system heterogeneity term which is further upper bounded using the bounded heterogeneity condition in \eqref{eq: intracluster heterogeneity}. The remaining of the proof follows exactly as in Lemma \ref{lemma: homogeneous-regression-error}, where the bounds for the statistical error \eqref{statistical error} and misclassification error \eqref{misclassification error} are leveraged to obtain the error bound presented in  \eqref{eq: error with heterogeneity}.
\end{proof}

\subsection{Adversarially Robust System Identification}

We now focus on the setting where adversarial systems may be present within the cluster of interest. To mitigate their effect in the multitask system identification process, we adopt an aggregation scheme that is $(f,\lambda)$-resilient, as defined in Definition~\ref{def: resilient aggregation}.

\begin{lemma}[Estimation error under $(f_j,\lambda)$-resilient aggregation]\label{lem:cluster-sysid-resilient}
Fix a cluster $\mathcal C_j$ with $M_j$ systems, among which at most $f_j<M_j/2$ are adversarial and $m_j:=M_j-f_j$ are honest.
For each honest system $i\in \mathcal C_j$, data generation follow \eqref{eq:data generation}. Suppose that the aggregation function $F\left(\left\{G_{\ell}(\widehat{\Theta}^{(i)})\right\}_{\ell \in \mathcal C_{j}}\right)$ is $(f,\lambda)$-resilient and that the intra-cluster heterogeneity is bounded by $\epsilon_{\mathrm{het}}$. Moreover, suppose Assumption \ref{assumption: initial model estimate} holds. Then, given a small probability $\delta \in (0,1)$, it holds that
\begin{align}\label{eq:cluster-bound}
\|\widehat{\Theta}^{(i)}_{N} - \Theta^{(i)}_\star\|^2_F &\leq  \frac{ C_{\mathrm{stat}}\sigma^2_w (d^2_x + d_x d_u)\log(1/\delta)}{\sigma^2_k \tau_k}\left( \frac{1}{m_{\hat{j}}} + \lambda^2 d_x\right) + C_{\mathrm{het}}(1+\lambda)^2 \epsilon^2_{\mathrm{het}}\notag \\
&+  C_{\mathrm{mis,1}} \exp(-C_{\mathrm{mis,2}} \sigma_k^2 \tau_k),
\end{align}
for any system $i \in \mathcal C_j$, with probability at least $1-\delta-\exp(-C_{\mathrm{mis,2}}\sigma_k^2 \tau_k)$. 
\end{lemma}

\begin{proof}
We begin the proof by recalling that, during epoch $k$ (of length $\tau_k$), each system $i \in [M]$ uses its collected dataset $\mathcal{D}^{(i)} = \{X^{(i)}, Z^{(i)}\}$ to perform model identification through clustered system identification. As described in Algorithm~\ref{algorithm: clustered sysid}, the state–input data are utilized both for estimating the cluster identity and for learning the model parameters. In the latter step, in particular, we store the gradient preconditioned by the inverse of the empirical state–input covariance matrix:
$$
G_{\ell}(\widehat{\Theta}^{(i)}) \leftarrow (X^{(\ell)} - \widehat{\Theta}^{(i)}Z^{(\ell)})Z^{(\ell)\top} \textcolor{black}{(Z^{(\ell)}Z^{(\ell)\top})^{-1}}
$$
of each system that has the same cluster identity as the $i$ system. The model parameter os system $i$ is then updated as follows: 

\begin{align}\label{per iteration model update}
    \widehat{\Theta}^{(i)}_{n+1} = \widehat{\Theta}^{(i)}_{n} + \eta \textcolor{black}{F(G_{\ell}(\widehat{\Theta}^{(i)}_{n}), \ldots, G_{i}(\widehat{\Theta}^{(i)}_{n}),\ldots, G_{s}(\widehat{\Theta}^{(i)}_{n}))},
\end{align}
where $\mathcal C_{\hat{j}} = \{\ell, \dots, i, \ldots, s\}$ is the set of indices of all systems that belong to the identified cluster $\hat{j}$ of system $i \in [M]$. From \eqref{per iteration model update} we can write
\begin{align*}
    \widehat{\Theta}^{(i)}_{n+1} &= \widehat{\Theta}^{(i)}_{n} + \eta \left( \textcolor{black}{\bar{G}} + F\left(\left\{G_{\ell}(\widehat{\Theta}^{(i)}_n)\right\}_{\ell \in \mathcal C_{\hat{j}}}\right) - \textcolor{black}{\bar{G}} \right)\\
    &=\widehat{\Theta}^{(i)}_{n} + \eta \bar{G} +  \eta \left(F\left(\left\{G_{\ell}(\widehat{\Theta}^{(i)}_n)\right\}_{\ell \in \mathcal C_{\hat{j}}}\right) - \bar{G} \right)\\
    &=\widehat{\Theta}^{(i)}_{n} +  \underbrace{\frac{\eta}{|\mathcal H_{\hat{j}}|} \sum_{\ell\in \mathcal H_{j} \cap \mathcal H_{\hat{j}}} G_l(\widehat{\Theta}^{(i)}_n)}_{\text{Correct classified honest systems}} + \underbrace{\frac{\eta}{|\mathcal H_{\hat{j}}|} \sum_{\ell\in \mathcal H^c_{j} \cap \mathcal H_{\hat{j}}} G_l(\widehat{\Theta}^{(i)}_n)}_{\text{Misclassified honest systems}} + \eta \underbrace{\left(F\left(\left\{G_{\ell}(\widehat{\Theta}^{(i)}_n)\right\}_{\ell \in \mathcal C_{\hat{j}}}\right) - \bar{G} \right)}_{\text{resilient aggregation error}},
\end{align*}
where $\bar{G} = \frac{1}{|\mathcal H_{\hat{j}}|} \sum_{\ell\in \mathcal H_{\hat{j}}} G_l(\widehat{\Theta}^{(i)}_n)$ denotes the average of honest system's gradient updates. In addition, $|\mathcal H_{j} \cap \mathcal H_{\hat{j}}|$ and $\mathcal H^c_{j} \cap \mathcal H_{\hat{j}}$ denote the number of honest systems that are correct classified to cluster $\hat{j}$ and misclassified to cluster $\hat{j}$, respectively. Here, $ \mathcal H^c_{j}$ denotes the complement of the set of honest systems in cluster $j \in [N_c]$. We proceed our analysis, by controlling the estimation error in the correct classified honest systems. For this purpose, let us define $S_1 = $ misclassified honest systems term $+$ resilient aggregation error term, and write the following:
\begin{align*}
    \widehat{\Theta}^{(i)}_{n+1} &=\widehat{\Theta}^{(i)}_{n} +  \frac{\eta}{|\mathcal H_{\hat{j}}|} \sum_{\ell\in \mathcal H_{j} \cap \mathcal H_{\hat{j}}} G_l(\widehat{\Theta}^{(i)}_n) + S_1\\
    &= \widehat{\Theta}^{(i)}_{n} +  \frac{\eta}{|\mathcal H_{\hat{j}}|} \sum_{\ell\in \mathcal H_{j} \cap \mathcal H_{\hat{j}}} (X^{(\ell)} - \widehat{\Theta}^{(i)}Z^{(\ell)})Z^{(\ell)\top} (Z^{(\ell)}Z^{(\ell)\top})^{-1} + S_1\\
    &=\widehat{\Theta}^{(i)}_{n} +  \frac{\eta}{|\mathcal H_{\hat{j}}|} \sum_{\ell\in \mathcal H_{j} \cap \mathcal H_{\hat{j}}} (\Theta^{(\ell)}_\star Z^{(\ell)} + W^{(\ell)} - \widehat{\Theta}^{(i)}Z^{(\ell)})Z^{(\ell)\top} (Z^{(\ell)}Z^{(\ell)\top})^{-1} + S_1\\
    &=\widehat{\Theta}^{(i)}_{n} +  \frac{\eta}{|\mathcal H_{\hat{j}}|} \sum_{\ell\in \mathcal H_{j} \cap \mathcal H_{\hat{j}}} \left(\left(\textcolor{black}{\Theta^{(i)}_\star} + \Theta^{(\ell)}_\star - \textcolor{black}{\Theta^{(i)}_\star} \right)Z^{(\ell)} + W^{(\ell)} - \widehat{\Theta}^{(i)}Z^{(\ell)}\right)Z^{(\ell)\top} (Z^{(\ell)}Z^{(\ell)\top})^{-1} + S_1\\
    &= \widehat{\Theta}^{(i)}_{n} +  \eta \left(\Theta^{(i)}_\star - \widehat{\Theta}^{(i)}_n \right) + \underbrace{\frac{\eta}{|\mathcal H_{\hat{j}}|} \sum_{\ell\in \mathcal H_{j} \cap \mathcal H_{\hat{j}}} \left(\Theta^{(\ell)}_\star - \Theta^{(i)}_\star\right)}_{\text{System heterogeneity}}\\
    &+ \frac{\eta}{|\mathcal H_{\hat{j}}|} \sum_{\ell\in \mathcal H_{j} \cap \mathcal H_{\hat{j}}} W^{(\ell)}Z^{(\ell)\top}(Z^{(\ell)}Z^{(\ell)\top})^{-1} + S_1.
\end{align*}

We now proceed to control $S_1$. To do so, let us first denote by $S_2$ the resilient aggregation error term, and write the following:

\begin{align*}
    S_1 &= \frac{\eta}{|\mathcal H_{\hat{j}}|} \sum_{\ell\in \mathcal H^c_{j} \cap \mathcal H_{\hat{j}}} G_l(\widehat{\Theta}^{(i)}_n) + S_2\\
    &= \frac{\eta}{|\mathcal H_{\hat{j}}|} \sum_{\ell\in \mathcal H^c_{j} \cap \mathcal H_{\hat{j}}} (\Theta_j Z^{(\ell)} + W^{(\ell)} - \widehat{\Theta}^{(i)}Z^{(\ell)})Z^{(\ell)\top} (Z^{(\ell)}Z^{(\ell)\top})^{-1} + S_2\\
    &= \underbrace{\frac{\eta}{|\mathcal H_{\hat{j}}|} \sum_{\ell\in \mathcal H^c_{j} \cap \mathcal H_{\hat{j}}} (\Theta_j - \widehat{\Theta}^{(i)})}_{\text{Misclassification error}} + \frac{\eta}{|\mathcal H_{\hat{j}}|} \sum_{\ell\in \mathcal H^c_{j} \cap \mathcal H_{\hat{j}}}  W^{(\ell)}Z^{(\ell)\top} (Z^{(\ell)}Z^{(\ell)\top})^{-1} +   S_2
\end{align*}
where $\Theta_j$ denotes the common model that system $q$ would use if it were correctly classified to its true cluster, which is different from $\hat{j}$. Therefore, we obtain 
\begin{align*}
    \widehat{\Theta}^{(i)}_{n+1} &= \widehat{\Theta}^{(i)}_{n} +  \eta \left(\Theta^{(i)}_\star - \widehat{\Theta}^{(i)}_n \right) + \frac{\eta}{|\mathcal H_{\hat{j}}|} \sum_{\ell\in \mathcal H_{\hat{j}}} \left(\Theta^{(\ell)}_\star - \Theta^{(i)}_\star\right)+  \frac{\eta|\mathcal H^c_j\cap \mathcal H_{\hat{j}}|}{|\mathcal H_{\hat{j}}|}(\Theta_j - \widehat{\Theta}^{(i)})\\
    &+  \frac{\eta}{|\mathcal H_{\hat{j}}|} \sum_{\ell\in \mathcal H_{\hat{j}}} W^{(\ell)}Z^{(\ell)\top}(Z^{(\ell)}Z^{(\ell)\top})^{-1} + S_2
\end{align*}
where we can subtract $\Theta^{(i)}_\star$ from both sides to write
\begin{align*}
    \widehat{\Theta}^{(i)}_{n+1} - \Theta^{(i)}_\star &= \widehat{\Theta}^{(i)}_{n} - \Theta^{(i)}_\star +  \eta \left(\Theta^{(i)}_\star - \widehat{\Theta}^{(i)}_n \right) + \frac{\eta}{|\mathcal H_{\hat{j}}|} \sum_{\ell\in \mathcal H_{\hat{j}}} \left(\Theta^{(\ell)}_\star - \Theta^{(i)}_\star\right)+ \frac{\eta|\mathcal H^c_j\cap \mathcal H_{\hat{j}}|}{|\mathcal H_{\hat{j}}|}(\Theta_j - \widehat{\Theta}^{(i)})\\
    &+  \frac{\eta}{|\mathcal H_{\hat{j}}|} \sum_{\ell\in \mathcal H_{\hat{j}}} W^{(\ell)}Z^{(\ell)\top}(Z^{(\ell)}Z^{(\ell)\top})^{-1} + S_2
\end{align*}
which implies
\begin{align}\label{eq: model convergence step}
    \|\widehat{\Theta}^{(i)}_{n+1} - \Theta^{(i)}_\star\|_F &\leq \left(1 - \eta \right)\|\widehat{\Theta}^{(i)}_{n} - \Theta^{(i)}_\star\|_F + \eta \epsilon_{\mathrm{het}}+ \underbrace{\frac{\eta|\mathcal H^c_j\cap \mathcal H_{\hat{j}}|}{|\mathcal H_{\hat{j}}|}\|\Theta_j - \widehat{\Theta}_\star^{(i)}\|_F}_{\text{misclassification error } \eqref{misclassification error}}\notag \\
    &+  \underbrace{\frac{\eta}{|\mathcal H_{\hat{j}}|} \sum_{\ell\in \mathcal H_{\hat{j}}} \|W^{(\ell)}Z^{(\ell)\top}(Z^{(\ell)}Z^{(\ell)\top})^{-1}\|_F}_{\text{statistical error bound } \eqref{statistical error}} + S_2.
\end{align}

As discussed previously in the proof of Lemma \ref{lemma: homogeneous-regression-error}, we have that 
\begin{align}\label{statistical error and misclassification rate}
    &\frac{\eta}{|\mathcal H_{\hat{j}}|} \sum_{\ell\in \mathcal H_{\hat{j}}} \|W^{(\ell)}Z^{(\ell)\top}(Z^{(\ell)}Z^{(\ell)\top})^{-1}\|_F  \leq \frac{\eta C_{\mathrm{stat}}\sigma_w \sqrt{(d^2_x + d_x d_u)\log(1/\delta)}}{\sigma_k \sqrt{\tau_k m_{\hat{j}}}},\\
    & \frac{\eta|\mathcal H^c_j\cap \mathcal H_{\hat{j}}|}{|\mathcal H_{\hat{j}}|}\|\Theta_j - \widehat{\Theta}_\star^{(i)}\|_F \leq \eta C_{\mathrm{mis,1}} \exp(-C_{\mathrm{mis,2}}\sigma_k^2  \tau_k).
\end{align}

We now proceed to control the resilient aggregation error. For this purpose, we will use 
Definition \ref{def: resilient aggregation} to write
\begin{align*}
    \|S_2\|_F \leq \eta \left\|F\left(\left\{G_{\ell}(\widehat{\Theta}^{(i)}_n)\right\}_{\ell \in \mathcal C_{\hat{j}}}\right) - \bar{G}\right\|_F\leq \eta \lambda \sqrt{d_x}  \max _{s, p \in \mathcal H_{\hat{j}}}\left\|G_{s}(\widehat{\Theta}^{(i)}_n)-G_{p}(\widehat{\Theta}^{(i)}_n)\right\|_F.
\end{align*}
where the dimensionality factor $\sqrt{d_x}$ is due to upper bounding the Frobenius norm with the spectral norm of the corresponding matrix. We then proceed by examining the sets of system indices that are correctly and incorrectly assigned to cluster $\hat{j}$. 
\begin{align*}
    \|S_2\|_F \leq  \eta \lambda \sqrt{d_x} \left( \max _{s, p \in \mathcal H_{\hat{j}} \cap \mathcal H_{{j}}}\left\|G_{s}(\widehat{\Theta}^{(i)}_n)-G_{p}(\widehat{\Theta}^{(i)}_n)\right\|_F + \max _{s, p \in \mathcal H_{\hat{j}} \cap \mathcal H^c_{{j}}}\left\|G_{s}(\widehat{\Theta}^{(i)}_n)-G_{p}(\widehat{\Theta}^{(i)}_n)\right\|_F\right),
\end{align*}
where we note that 
\begin{align}\label{eq: gradient deviation}
 \left\|G_{s}(\widehat{\Theta}^{(i)}_n)-G_{p}(\widehat{\Theta}^{(i)}_n)\right\|_F &\leq \left\|\Theta^{(s)}_\star -  \Theta^{(p)}_\star\right\|_F +  \left\|W^{(s)}Z^{(s)\top}(Z^{(s)}Z^{(s)\top})^{-1}\right\|_F\notag \\
 &+ \left\|W^{(p)}Z^{(p)\top}(Z^{(p)}Z^{(p)\top})^{-1}\right\|_F\notag \\
&\leq \epsilon_{\mathrm{het}} + \frac{2 \sqrt{d_x} C_{\mathrm{stat}}\sigma_w \sqrt{(d^2_x + d_x d_u)\log(1/\delta)}}{\sigma_k \sqrt{\tau_k}}.
\end{align}

It is worth noting that the first term corresponds to the intra-cluster system heterogeneity. This term arises in the adversarial aggregation error, as the server cannot distinguish between deviations caused by adversarial attacks and those due to natural heterogeneity. It reflects a fundamental limit of adversarially robust learning under heterogeneous systems, consistent with the lower bound established in~\cite{karimireddy2020byzantine} for the Byzantine federated learning setting.

Therefore, by plugging \eqref{eq: gradient deviation} into \eqref{eq: model convergence step}, we obtain
\begin{align*}
    \|\widehat{\Theta}^{(i)}_{n+1} - \Theta^{(i)}_\star\|_F &\leq \left(1 - \eta \right)\|\widehat{\Theta}^{(i)}_{n} - \Theta^{(i)}_\star\|_F + \eta(1+2\lambda)\epsilon_{\mathrm{het}}+ \eta C_{\mathrm{mis,1}} \exp(-C_{\mathrm{mis,2}} \sigma_k^2 \tau_k)\\
    &+  \frac{\eta C_{\mathrm{stat}}\sigma_w \sqrt{(d^2_x + d_x d_u)\log(1/\delta)}}{\sigma_k \sqrt{\tau_k m_{\hat{j}}}} + \frac{4\eta \lambda \sqrt{d_x} C_{\mathrm{stat}}\sigma_w \sqrt{(d^2_x + d_x d_u)\log(1/\delta)}}{\sigma_k \sqrt{\tau_k}}.
\end{align*}

Therefore, by unrolling the above expression over $N$ iterations we obtain 
\begin{align*}
    \|\widehat{\Theta}^{(i)}_{N} - \Theta^{(i)}_\star\|_F &\leq \rho^N\|\widehat{\Theta}^{(i)}_{n} - \Theta^{(i)}_\star\|_F + (1+2\lambda) \epsilon_{\mathrm{het}}+ C_{\mathrm{mis,1}} \exp(-C_{\mathrm{mis,2}} \sigma_k^2 \tau_k)\\
    &+  \frac{ C_{\mathrm{stat}}\sigma_w \sqrt{(d^2_x + d_x d_u)\log(1/\delta)}}{\sigma_k \sqrt{\tau_k m_{\hat{j}}}} + \frac{\lambda \sqrt{d_x} C_{\mathrm{stat}}\sigma_w \sqrt{(d^2_x + d_x d_u)\log(1/\delta)}}{\sigma_k \sqrt{\tau_k}}.
\end{align*}
with $\rho = 1-\eta$. Note that any additional constant factors are absorbed by $C_{\mathrm{stat}}$, $C_{\mathrm{mis,1}}$ and $C_{\mathrm{mis,2}}$. Moreover, we can set the total number of iterations as $N \geq \log\left(C_\alpha \Delta_{\text{min}}\tau^2_1\right)/\log(1/\rho),$
since the initial model estimate is as in Assumption \ref{assumption: initial model estimate}. Therefore, the contraction term is of order $\mathcal{O}(1/\tau_k^2)$ and becomes negligible in the subsequent regret analysis. We then obtain the following estimation error bound:

\begin{align*}
    \|\widehat{\Theta}^{(i)}_{N} - \Theta^{(i)}_\star\|^2_F &\leq  \frac{ C_{\mathrm{stat}}\sigma^2_w (d^2_x + d_x d_u)\log(1/\delta)}{\sigma^2_k \tau_k}\left( \frac{1}{m_{\hat{j}}} + \lambda^2 d_x\right)\\
    &+ C_{\mathrm{het}}(1+\lambda)^2\epsilon^2_{\mathrm{het}}+  C_{\mathrm{mis,1}} \exp(-C_{\mathrm{mis,2}} \sigma_k^2\tau_k),
\end{align*}
for some universal constant $C_{\mathrm{het}}$. We complete the proof by noting that the misclassification rate decays exponentially with the epoch length. By choosing a sufficiently large initial epoch length $\tau_0$, and doubling the epoch length for subsequent epochs, we ensure that the misclassification rate becomes negligible, guaranteeing that the systems are correctly classified. Consequently, we have $\hat{j} \approx j$, and therefore $m_{\hat{j}} \approx m_{j}$.
\end{proof}

\section{Characterizing the Probability of the Success Event}  \label{appendix: probability of success}

We now demonstrate that the success event holds with high probability, that is, the event in which Algorithm~\ref{alg: CE with online clustering} does not abort and the estimation error bounds for Cases 1, 2, and 3 hold, as given in \eqref{eq: Eest1}, \eqref{eq: Eest2}, and \eqref{eq: Eest3}, respectively. Below, we establish the probability of success for Case 1 and note that the proof can be readily extended to Cases 2 and 3 by incorporating the additional heterogeneity bias and adversarial-aware terms that appear in those settings.

\begin{lemma}\label{lem: probability of success event}
Running Algorithm \ref{alg: CE with online clustering} with the arguments defined in Lemma \ref{lemma: homogeneous-regression-error}, the event $\mathcal{E}^{(j)}_{\mathsf{success}}$, for $j = 1,2,3$ holds with probability at least $1-T^{-2}$.
\end{lemma}

\begin{proof}
We show that the success event $\mathcal{E}^{(1)}_{\text{success}}$ holds with probability at least $1 - T^{-2}$ by induction. In particular, we prove that for every epoch $k \in [k_{\text{fin}}]$, Algorithm~\ref{alg: CE with online clustering} does not abort and the estimation error remains bounded as specified in $\mathcal{E}_{\text{est},1}$. The base case corresponds to the first epoch.

\renewcommand{\thefootnote}{\arabic{footnote}}

\noindent \textbf{Base case:} For convenience we assume that $x^{(i)}_0 = 0$, for all systems $i \in [M]$\footnote{ This proof can also be extended to bounded non-zero initial states.}.Note that we can use Lemma \ref{lem: noise bound} to obtain
 \begin{align}\label{eq:noise_bound_mt}
         \max_{0\leq t \leq T-1} \norm{\bmat{w^{(i)}_t \\ g^{(i)}_t}} \leq 4\sigma_w\sqrt{3 (d_x + d_u) \log (2MT)}, 
    \end{align}
with probability $1 - \frac{1}{2}T^{-2}$, for all systems $i \in [M]$. As our initial state is zero, we can guarantee that
$$\norm{x^{(h)}_1} \leq 16 (P_0^{\vee})^{3/2} \Psi_B^{\vee} \max_{0\leq t \leq T-1} \norm{\bmat{w^{(i)}_t \\ g^{(i)}_t}},$$ 
where $P_0^{\vee} \triangleq \max _{i \in [M]}\left\|P_{K_0^{(i)}}^{(i)}\right\|$, $P_{\star}^{\wedge} \triangleq \min _{i \in [M]}\left\|P_{\star}^{(i)}\right\|$, and $\Psi_B^{\vee} \triangleq \max_{i \in [M]} \Psi^{(i)}_{B}$. Therefore, by setting the initial epoch length according to $\tau_1 \geq \frac{ c \log \frac{1}{P_{\star}^{\wedge}}}{\log\left(1- \frac{1}{ P_{\star}^{\wedge}}\right)}$, for a sufficiently large constant $c$, we can use Lemma \ref{lem: state rollout bounds} to obtain
\begin{align}\label{eq:state_norm_first_epoch}
        \norm{x^{(i)}_t} \leq 40 (P_0^{\vee})^2\Psi_B^{\vee} \max_{0\leq t \leq T-1} \norm{\bmat{w^{(i)}_t \\ g^{(i)}_t}}, \forall t = \{0,1,\ldots,\tau_1\},
\end{align}
where we use \eqref{eq:noise_bound_mt} to obtain
\begin{align*}
        \norm{x^{(i)}_t}^2 \leq 76800 (P_0^{\vee})^4 (\Psi_B^{\vee})^2 \sigma^2_w (d_x + d_u)\log (2MT), \;\ \forall t = \{0,1,\ldots,\tau_1\}
\end{align*}
with probability $1- \frac{1}{2}T^{-2}$, for all systems $ i \in [M]$, which implies that  $\norm{x^{(i)}_t}^2 \leq x^2_b\log T$ which satisfies the state norm requirement to not abort as described in Algorithm \ref{alg: CE with online clustering}. On the other hand, to verify that the controller norm requirement is satisfied we note that $\|K^{(i)}_0\|^2 \leq P_0^{\vee} \leq 2P_0^{\vee}$, which leads to $\|K^{(i)}_0\| \leq K_b$.  Therefore, we have that $\mathcal{E}_{\text{bound}}$ holds with probability $1-\frac{1}{2}T^{-2}$.

We now proceed to control the estimation error at the first epoch. Note that by making $\tau_{1}$ sufficiently large such that $\tau_{1} \geq c(1+\log(4T^2))$, for a sufficiently large constant $c$, we can guarantee the following upper bound for the model estimation:

\begin{align*}
        &\norm{\bmat{\hat A^{(i)}_1 & \hat B^{(i)}_1} - \bmat{A^{(i)}_\star & B^{(i)}_\star}}_F^2 \leq \frac{C_{\mathrm{stat}}\sigma^2_w (d^2_x + d_x d_u) \log(2T^2)}{\sigma^2_1 M_{j} \tau_{1}} + C_{\mathrm{mis,1}} \exp\left(-C_{\mathrm{mis,2}} \sigma_1^2\tau_1\right)
    \end{align*}
for any system $i \in C_{j}$ as discussed in Lemma \ref{lemma: homogeneous-regression-error}.
Our induction step follows by considering the following inductive hypothesis:
\begin{align}\label{eq:inductive_hypotesis_state_control_bounds}
   \text{\textbf{Bounded state}:} \norm{x^{(i)}_{\tau_k}} \leq 16 (P_0^{\vee})^{3/2} \Psi_B^{\vee} \max_{0\leq t \leq T-1} \norm{\bmat{w^{(i)}_t \\ g^{(i)}_t}},
\end{align}
and
\begin{align}\label{eq:inductive_hypotesis_estimation_error}
       &\text{\textbf{Estimation error}:}\notag\\
       &\norm{\bmat{\hat A^{(i)}_k & \hat B^{(i)}_k} - \bmat{A^{(i)}_\star & B^{(i)}_\star}}_F^2 \leq \frac{C_{\mathrm{stat}}\sigma^2_w (d^2_x + d_x d_u) \log(2T^2)}{\sigma^2_k M_{j} \tau_{k}} + C_{\mathrm{mis,1}} \exp\left(-C_{\mathrm{mis,2}} \sigma_k^2\tau_k\right),
\end{align}
where we demonstrate that the estimation error holds for the next epoch $k+1$ by balancing the exploration and exploitation as $\tau_k \sigma^2_k \geq \frac{1}{2} \tau_{k+1}\sigma^2_{k+1}$. In addition, we guarantee that the state and controller norm bounds are not violated by combining the fact that $\norm{P^{(i)}_{\hat K_{k+1}}}  \leq  2 (P_0^{\vee})$ and $\tau_k \geq \tau_1 \geq \frac{ c \log \frac{1}{P_{\star}^{\wedge}}}{\log\left(1- \frac{1}{ P_{\star}^{\wedge}}\right)}$, for a sufficiently large constant $c$, along with Lemma \ref{lem: state rollout bounds} to  obtain 
\begin{align}\label{eq:state_norm_first_epoch_1}
        \norm{x^{(i)}_t} \leq 40 (P_0^{\vee})^2(\Psi_B^{\vee}) \max_{1\leq t \leq T} \norm{\bmat{w^{(i)}_t \\ g^{(i)}_t}}, \;\ \forall t = \{\tau_k+1,\ldots,\tau_{k+1}\},
\end{align}
which guarantees that the state norm requirement is satisfied. For the controller norm, we have that $\norm{\hat K^{(i)}_{k+1}}^2 \leq \norm{P_{\hat K_{k+1}}} \leq 2 P_0^{\vee}$, and thus $\norm{\hat K^{(i)}_{k+1}} \leq K_b$.  We conclude the proof by noting that for epoch $k$ the bounded state and controller norm requirements hold with probability $1 - \frac{1}{2}T^{-2}$, then for epoch $k+1$ also holds with at least the same probability. Then, by union bounding for all the epochs, we have that $\mathcal{E}_{\mathsf{success}}$ holds under probability of at least $1 - T^{-2}.$
\end{proof}

\section{Regret Analysis} \label{appendix: regret analysis}

With the estimation error bounds and the probability of the success event in place, we are now ready to synthesize the regret bounds for the three settings under consideration: (i) no intra-cluster heterogeneity, (ii) intra-cluster heterogeneity, and (iii) the presence of adversarial systems. In this section, we leverage the adapted regret bounds from~\cite{lee2023nonasymptotic}, revisited in Lemmas~\ref{lem:R1_cluster}, \ref{lem:R2-bound}, and~\ref{lem:R3-bound}, and derive conditions on the exploration sequence $\{\sigma_k\}_k$ and the epoch length $\tau_k$ (i.e., exploitation phase) that ensure a reduction of the leading term in the regret proportional to the number of honest systems participating in the collaboration. For completeness, we restate here the key theorems presented in the main body of the paper.

\subsection{Regret with Intra-cluster Homogeneity}

We begin with the setting in which the models within each cluster $\mathcal{C}_j$, for any $j \in [N_c]$, are identical, and each cluster contains $M_j$ models.

\begin{theorem}[Intra-cluster homogeneity]
\label{thm:intra-cluster-regret-homo}
Fix a system $i \in [M]$ that belongs to a homogeneous cluster $\mathcal C_j$ of size $M_j$. Let the assumptions of Algorithm~\eqref{algorithm: clustered sysid} hold. Suppose the exploration sequence satisfies
$\sigma_k^2 = \frac{\sqrt{d^2_u d_x}}{(d_x^2 + d_xd_u)\sqrt{\tau_k M_j}}$ and the epoch length doubles, i.e., $\tau_{k}=2^{k-1}\tau_1$ with $T=\tau_{k_{\mathrm{fin}}}$.
Then the expected regret of system $i$ satisfies
\begin{align}\label{regret-no-het}
   \mathbb{E}\big[\mathcal{R}_T^{(i)}\big] \leq \Omega_1 \sqrt{\frac{d^2_u d_x T}{M_j}} + \Omega_2 (\log T)^2 + \Omega_3 T \exp\left(-\frac{C_{\mathrm{mis,2}} \sqrt{\tau_1}}{\sqrt{M_j}}\right), 
\end{align}
where
\begin{align*}
    \Omega_1 \;=&\; 142C_{\texttt{stat}} \|P^{(i)}_\star\|^{8}\sigma^2_w \log(1/\delta) + 2 d_u + 4 d_u \|P^{(i)}_\star\| \Psi^{(i)2}_{B},\\
\Omega_2 \;=&\; 3\,\max\{d_x,d_u\}\,\|P^{(i)}_{K^{(i)}_0}\|\,\Psi^{(i)2}_{B} + 2x_b^2\,\|P^{(i)}_{K_\star}\|,\\
\Omega_3 \;=&\; 142C_{\texttt{mis,1}}\|P^{(i)}_\star\|^{8}. 
\end{align*}
\end{theorem}

All constants above are as defined in the lemmas referenced in the proof.

\begin{proof}
By Lemma~\ref{lem:R1_cluster}, Lemma~\ref{lem:R2-bound}, Lemma~\ref{lem:R3-bound}, and the regret decomposition in
\eqref{eq: regret decomposition appendix}, we have
\begin{align*}
\mathbb{E}\big[\mathcal{R}_T^{(i)}\big]
&\le \sum_{k=2}^{k_{\mathrm{fin}}}\Bigg(
\mathbb{E}\!\Big[\mathbf{1}(\mathcal{E}_{\mathrm{est},1}^{(k-1)})\; 142\,(\tau_k-\tau_{k-1})\,\|P^{(i)}_\star\|^{8}\,
\big\|\,[\widehat A^{(i)}_{k-1}\;\widehat B^{(i)}_{k-1}]-[A^{(i)}_\star\;B^{(i)}_\star]\big\|_F^2\Big]  \\
&+ (\tau_k-\tau_{k-1})\,J^{(i)}(K^{(i)}_\star)
+ 4(\tau_k-\tau_{k-1})\,d_u\,\|P^{(i)}_\star\|\sigma_k^2\Psi^{(i)2}_{B}
+ 2x_b^2\log T\,\|P^{(i)}_\star\|
\Bigg) \\
&+ T^{-1} \big( \|Q\| + 2K_b^2 \big) x_b^2 \log T + T^{-1} J(K_0) \\
&+ 24 \|P^{(i)}_{K^{(i)}_0}\| \Psi^{(i)2}_{B} (d_x + d_u)\, \sigma^2_w T^{-2}\log(3T)
+ 2 T^{-2} \|P^{(i)}_{K^{(i)}_0}\| \|\Theta^{(i)}_\star\|_F^2 K_b^2 x_b^2 \log T  \\
& + \sum_{k=1}^{k_{\mathrm{fin}}} 2(\tau_k - \tau_{k-1}) d_u \,\sigma_k^2
+ 3 \tau_1 \max\{ d_x , d_u \} \| P^{(i)}_{K^{(i)}_0} \| \Psi_{B_\star}^2 - T J(K^{(i)}_\star).
\end{align*}

In the homogeneous case and by \eqref{eq: Eest1}, for each epoch $k\ge 2$,
\begin{align*}
\mathbb{E}\left[\mathbf{1}(\mathcal{E}_{\mathrm{est},1}^{(k-1)})\,
\big\|\,[\widehat A^{(i)}_{k-1}\;\widehat B^{(i)}_{k-1}]-[A^{(i)}_\star\;B^{(i)}_\star]\,\big\|_F^2\right]
&\le \frac{C_{\mathrm{stat}}\sigma^2_w (d^2_x + d_x d_u) \log(1/\delta)}{\sigma^2_{k-1} M_{j} \tau_{k-1}}\\
&+ C_{\mathrm{mis,1}} \exp\left(-C_{\mathrm{mis,2}}\, \sigma_{k-1}^2 \,\tau_{k-1}\right).    
\end{align*}

Substituting this bound yields
\begin{align*}
\mathbb{E}\big[\mathcal{R}_T^{(i)}\big]
&\le \sum_{k=2}^{k_{\mathrm{fin}}}\Bigg(
142(\tau_k-\tau_{k-1})\|P^{(i)}_\star\|^{8}
\Big(\frac{C_{\mathrm{stat}}\sigma^2_w (d^2_x + d_x d_u) \log(1/\delta)}{\sigma^2_{k-1} M_{j} \tau_{k-1}} \\
&+ C_{\mathrm{mis,1}} \exp\left(-C_{\mathrm{mis,2}}\, \sigma_{k-1}^2 \tau_{k-1}\right)\Big) +  (\tau_k-\tau_{k-1})\,J(K^\star)
+ 4(\tau_k-\tau_{k-1})\,d_u\,\|P^{(i)}_\star\|\,\sigma_k^2\,\Psi^{(i)2}_{B} \\
& + 2x_b^2\log T\,\|P^{(i)}_\star\|
\Bigg) + T^{-1} \big( \|Q\| + 2K_b^2 \big) x_b^2 \log T + T^{-1} J^{(i)}(K^{(i)}_0) \\
&+ 24 \|P^{(i)}_{K^{(i)}_0}\| \Psi^{(i)2}_{B} (d_x + d_u)\, \sigma^2_w T^{-2}\log(3T)
+ 2 T^{-2} \|P^{(i)}_{K^{(i)}_0}\| \|\Theta^{(i)}_\star\|_F^2 K_b^2 x_b^2 \log T  \\
&+ \sum_{k=1}^{k_{\mathrm{fin}}} 2(\tau_k - \tau_{k-1}) d_u \,\sigma_k^2
+ 3 \tau_1 \max\{ d_x , d_u \} \| P^{(i)}_{K^{(i)}_0} \| \Psi^{(i)2}_{B^{(i)}_\star} - T J^{(i)}(K_\star).
\end{align*}

Assume $\sigma_k^2 = \frac{\sqrt{d^2_u d_x}}{(d_x^2 + d_xd_u)\sqrt{\tau_k M_j}}$ and $\tau_k = 2^{k-1}\tau_1$.
Since we double the epoch length over the epochs, i.e., $\tau_k$ increasing, we have that $e^{-C_{\mathrm{mis,2}}\sqrt{\tau_k}}\le e^{-C_{\mathrm{mis,2}}\sqrt{\tau_1}}$. Using these conditions, we obtain
\begin{align*}
\mathbb{E}\big[\mathcal{R}_T^{(i)}\big]
&\le \sum_{k=2}^{k_{\mathrm{fin}}}\Bigg(
142(\tau_k-\tau_{k-1})\|P^{(i)}_\star\|^{8}
\Big(\frac{C_{\mathrm{stat}}\sigma^2_w \sqrt{d^2_u d_x} \log(1/\delta)}{\sqrt{M_{j} \tau_{k-1}}} + C_{\mathrm{mis,1}} \exp \left(-\frac{C_{\mathrm{mis,2}} \sqrt{\tau_1}}{\sqrt{M_j}}\right)
\\
&+ (\tau_k-\tau_{k-1})J^{(i)}(K^{(i)}_\star)
+ 4(\tau_k-\tau_{k-1})d_u \sqrt{d^2_u d_x}\|P^{(i)}_\star\|\frac{\Psi^{(i)2}_{B}}{\sqrt{\tau_k M_j}}
+ 2x_b^2\log T\,\|P^{(i)}_\star\|
\Bigg) \\
&+ T^{-1} \big( \|Q\| + 2K_b^2 \big) x_b^2 \log T + T^{-1} J^{(i)}(K^{(i)}_0) \\
&+ 24 \|P^{(i)}_{K^{(i)}_0}\| \Psi^{(i)2}_{B} (d_x + d_u) \sigma^2_w T^{-2}\log(3T)
+ 2 T^{-2} \|P^{(i)}_{K^{(i)}_0}\| \|\Theta^{(i)}_\star\|_F^2 K_b^2 x_b^2 \log T  \\
&+ \sum_{k=1}^{k_{\mathrm{fin}}} 2(\tau_k - \tau_{k-1}) d_u \sqrt{d^2_u d_x} \,\frac{1}{\sqrt{\tau_k M_i}}
+ 3 \tau_1 \max\{ d_x , d_u \} \| P^{(i)}_{K^{(i)}_0} \| \Psi^{(i)2}_{B} - T J^{(i)}(K^{(i)}_\star).
\end{align*}

Note that $\sum_{k=2}^{k_{\mathrm{fin}}}(\tau_k-\tau_{k-1}) = \tau_{k_{\mathrm{fin}}}-\tau_1 = T-\tau_1$ and
$\sum_{k=2}^{k_{\mathrm{fin}}} 1 = k_{\mathrm{fin}}-1 \asymp \log T$.
Moreover, a standard estimate for the doubling epoch length gives
$$
\sum_{k=2}^{k_{\mathrm{fin}}} \frac{\tau_k-\tau_{k-1}}{\sqrt{\tau_{k-1}}}
\leq 2\big(\sqrt{T}-\sqrt{\tau_1}\big) \;\le\; 2\sqrt{T}.
$$

By using these facts and separating terms, we obtain
\begin{align*}
\mathbb{E}\big[\mathcal{R}_T^{(i)}\big]
&\le 142\|P^{(i)}_\star\|^{8} C_{\mathrm{stat}}\sigma^2_w \sqrt{d^2_u d_x} \log(1/\delta)
\sum_{k=2}^{k_{\mathrm{fin}}} \frac{\tau_k-\tau_{k-1}}{\sqrt{\tau_{k-1} M_j}}
+ 2 d_u \sqrt{d^2_u d_x} \sum_{k=1}^{k_{\mathrm{fin}}} \frac{\tau_k-\tau_{k-1}}{\sqrt{\tau_{k} M_j}} \\
&+ 4 d_u \sqrt{d^2_u d_x} \|P^{(i)}_\star\| \Psi^{(i)2}_{B}
\sum_{k=2}^{k_{\mathrm{fin}}} \frac{\tau_k-\tau_{k-1}}{\sqrt{\tau_k M_j}}
+ 2x_b^2\|P^{(i)}_\star\|(\log T)\sum_{k=2}^{k_{\mathrm{fin}}} 1 \\
&+ 142 \|P^{(i)}_\star\|^{8} C_{\mathrm{mis,1}} \exp\left(-\frac{C_{\mathrm{mis,2}} \,\sqrt{\tau_1}}{\sqrt{M_j}}\right) \sum_{k=2}^{k_{\mathrm{fin}}} (\tau_k-\tau_{k-1}) \\
&+ T^{-1} \big( \|Q\| + 2K_b^2 \big) x_b^2 \log T + T^{-1} J^{(i)}(K^{(i)}_0) \\
&+ 24 \|P^{(i)}_{K^{(i)}_0}\| \Psi^{(i)2}_{B} (d_x + d_u)\, \sigma^2_w T^{-2}\log(3T)
+ 2 T^{-2} \|P^{(i)}_{K^{(i)}_0}\| \|\Theta^{(i)}_\star\|_F^2 K_b^2 x_b^2 \log T \\
&+ \big(T-\tau_1\big)J^{(i)}(K^{(i)}_\star) - T J(K^{(i)}_\star)
+ 3 \tau_1 \max\{ d_x , d_u \} \| P^{(i)}_{K^{(i)}_0} \| \Psi^{(i)2}_{B}.
\end{align*}

Therefore, since $(T-\tau_1)J(K^\star) - T J(K^\star) = -\tau_1 J(K^\star) \le 0$, we may drop this (non-positive) term. By applying $\sum_k (\tau_k-\tau_{k-1})/\sqrt{\tau_k}\le 2\sqrt{T}$,
$\sum_{k=2}^{k_{\mathrm{fin}}}1\lesssim \log T$, and $\sum_{k=2}^{k_{\mathrm{fin}}} (\tau_k-\tau_{k-1})=T-\tau_1\le T$, we obtain

\begin{align*}
\mathbb{E}\big[\mathcal{R}_T^{(i)}\big]
&\leq 
\Big(142\|P^{(i)}_\star\|^{8} C_{\mathrm{stat}}\sigma^2_w (d^2_x + d_x d_u) \log(1/\delta) \sigma^2 + 2 d_u + 4 d_u \|P^{(i)}_{\star}\| \Psi_{B^\star}^2\Big)\sqrt{\frac{T}{M_i}}\\
&+\Big(3 \max\{ d_x , d_u \} \| P_{K_0} \| \Psi_{B_\star}^2 + 2x_b^2\|P^{(i)}_\star\|\Big)(\log T)^2 \\
&+ 142\,\|P^{(i)}_\star\|^{8} C_{\mathrm{mis,1}} \exp\left(-\frac{C_{\mathrm{mis,2}} \,\sqrt{\tau_1}}{\sqrt{M_j}}\right) T,
\end{align*}
which is precisely the claimed bound with the stated $\Omega_1$, $\Omega_2$,and $\Omega_3$. We also emphasize that by choosing the initial epoch length as $\tau_1 \ge M_j \log({C_{\mathrm{mis,1}}}T^2)/C^2_{\texttt{mis,2}}$, so that
$T e^{-C_{\mathrm{mis,2}}\tau_1}$ is negligible (i.e., $\mathcal{O}(1)$) and the $T^{-1}$ and $T^{-2}$ terms vanish as $T$ grows. Absorbing the remaining $\tau_1$-dependent constant into the $(\log T)^2$ term yields \eqref{regret-no-het} without the misclassification term. 
\end{proof}

\subsection{Regret with Intra-cluster Heterogeneity}

Next, we consider the setting where the models within the cluster $\mathcal{C}_j$ are similar but not identical, with bounded heterogeneity quantified by $\epsilon_{\mathrm{het}}$ (see Assumption~\ref{assumption: heterogeneity}).

\begin{corollary}[Intra-cluster heterogeneity]
\label{thm:intra-cluster-regret-het}
Fix a system $i$ that belongs to a heterogeneous cluster $\mathcal{C}_j$ of size $M_j$. 
Let the assumptions of Algorithm~\eqref{algorithm: clustered sysid} hold.  Suppose the exploration sequence satisfies
$\sigma_k^2 = \frac{\sqrt{d^2_u d_x}}{(d_x^2 + d_xd_u)\sqrt{\tau_k M_j}}$ and the epoch length doubles, i.e., $\tau_{k}=2^{k-1}\tau_1$ with $T=\tau_{k_{\mathrm{fin}}}$.
Then the expected regret of system $i$ satisfies
\begin{align}
\mathbb{E}\big[\mathcal{R}_T^{(i)}\big] \leq \Omega_1 \sqrt{\frac{d^2_u d_x T}{M_j}} + \Omega_2 \,(\log T)^2 + \Omega_3 T\exp\left(-\frac{C_{\mathrm{mis,2}} \,\sqrt{\tau_1}}{\sqrt{M_j}}\right) + \Omega_4 T \epsilon^2_{\mathrm{het}},    
\end{align}
with additional constant $\Omega_4 =  142C_\textit{het} \|P^{(i)}_\star\|^{8}$. 
\end{corollary}
\begin{proof}
The proof follows directly from the derivations established in Theorem~\ref{thm:intra-cluster-regret-homo}. We recall that the expected regret can be decomposed as follows
\begin{align*}
\mathbb{E}\big[\mathcal{R}_T^{(i)}\big]
&\le \sum_{k=2}^{k_{\mathrm{fin}}} \Bigg(
\mathbb{E}\!\Big[
\mathbf{1}\big(\mathcal{E}_{\mathrm{est},2}^{(k-1)}\big)\;
142\,(\tau_k-\tau_{k-1})\,\|P^{(i)}_\star\|^{8}\,
\big\|\,[\widehat A^{(i)}_{k-1}\;\widehat B^{(i)}_{k-1}] - [A^{(i)}_\star\;B^{(i)}_\star]\,\big\|_F^2
\Big]  \\
&+ (\tau_k-\tau_{k-1})\,J^{(i)}(K^{(i)}_\star)
+ 4(\tau_k-\tau_{k-1})\,d_u\,\|P^{(i)}_\star\|\,\sigma_k^2\,\Psi^{(i)2}_{B}
+ 2x_b^2\log T\,\|P^{(i)}_\star\|
\Bigg) \\
& + T^{-1} \big( \|Q\| + 2K_b^2 \big) x_b^2 \log T + T^{-1} J^{(i)}(K^{(i)}_0) \\
&+ 24 \|P^{(i)}_{K^{(i)}_0}\| \Psi^{(i)2}_{B} (d_x + d_u) \sigma^2_w T^{-2}\log(3T)
+ 2 T^{-2} \|P^{(i)}_{K^{(i)}_0}\| \|\Theta^{(i)}_\star\|_F^2 K_b^2 x_b^2 \log T  \\
& + \sum_{k=1}^{k_{\mathrm{fin}}} 2(\tau_k - \tau_{k-1}) d_u \sigma_k^2
+ 3 \tau_1 \max\{ d_x , d_u \} \| P^{(i)}_{K^{(i)}_0} \| \Psi^{(i)2}_{B} - T J^{(i)}(K^{(i)}_\star).
\end{align*}

We note that in the heterogeneous setting, the only additional contribution arises from the deviation between the true system parameters within each cluster (i.e., intra-cluster heterogeneity).  This effect appears in the first term through the heterogeneity bound, which introduces an additional bias term proportional to $\epsilon_{\mathrm{het}}^2$.  
Consequently, the regret bound becomes
\begin{align*}
\mathbb{E}\big[\mathcal{R}_T^{(i)}\big]
&\le \Omega_1 \sqrt{\frac{d^2_u d_x T}{M_j}} + \Omega_2 \,(\log T)^2 + \Omega_3 T\exp\left(-\frac{C_{\mathrm{mis,2}} \,\sqrt{\tau_1}}{\sqrt{M_j}}\right)\\
&+ \sum_{k=2}^{k_{\mathrm{fin}}}
142\,(\tau_k-\tau_{k-1})\,\|P^{(i)}_\star\|^{8} C_{\mathrm{het}}\epsilon_{\mathrm{het}}^2.
\end{align*}

By evaluating the summation over epochs and applying the doubling schedule $\tau_k = 2^{k-1}\tau_1$, we obtain the final bound
\begin{align*}
\mathbb{E}\big[\mathcal{R}_T^{(i)}\big]
&\leq
\Omega_1 \sqrt{\frac{d^2_u d_x T}{M_j}} \;+\; \Omega_2 \,(\log T)^2 + \Omega_3 T\exp\left(-\frac{C_{\mathrm{mis,2}} \sqrt{\tau_1}}{\sqrt{M_j}}\right)
+
142\,\|P^{(i)}_\star\|^{8}\,C_{\mathrm{het}} T \epsilon_{\mathrm{het}}^2,
\end{align*}
which completes the proof.
\end{proof}

\subsection{Regret with Adversarial Systems}

We now consider the setting where a given cluster $\mathcal{C}_j$ consists of $f_j$ adversarial systems and $m_j = M_j - f_j$ honest systems. The effect of the adversarial systems on the regret bounds manifests through the resilient coefficient $\lambda$ of the underlying resilient aggregation rule (see Definition~\ref{def: resilient aggregation}). Well-known resilient aggregation rules such as the coordinate-wise trimmed mean (CWTM) and the mean-around-median (MeaMed) typically exhibit a resilient coefficient on the order of $\mathcal{O}(f_j/m_j)$. We refer the reader to \cite{farhadkhani2022byzantine} for more details on such resilient aggregators. In Section~\ref{sec:numerics}, we further illustrate our adversarially robust adaptive control approach using GM, whose resilience coefficient scales as $\mathcal{O}\!\left(1+\frac{m_j}{\sqrt{M_j - 2f_j}\,M_j}\right)$. In this case, the improvement in regret with respect to the number of honest systems within each cluster may be attenuated.

\begin{theorem}[Adversarial and intra-cluster heterogeneity]
\label{thm:intra-cluster-regret-adv-het}
Consider a system $i$ belonging to a heterogeneous cluster $C_j$ of size $M_j$ that may contain a fraction of adversarial systems.
Let the assumptions of Algorithm~\eqref{algorithm: clustered sysid} hold. Suppose that the exploration sequence satisfies $\sigma_k^2 = \frac{\sqrt{d^2_u d_x}}{d_x^2 + d_xd_u}\sqrt{\frac{1 + \lambda^2 d_x m_{{j}}}{\tau_k m_{{j}}}}$ and that the epoch length is double every epoch, i.e., 
$\tau_{k} = 2^{k-1}\tau_1$, with total horizon $T = \tau_{k_{\mathrm{fin}}}$. Then, the expected regret of system $i$ satisfies
\begin{align*}
\mathbb{E}\big[\mathcal{R}_T^{(i)}\big]
&\leq
\Omega_1 \sqrt{\frac{d^2_u d_x T (1 + \lambda^2 d_x m_{j})}{m_{j}}} +  \Omega_2 (\log T)^2 + \Omega_3 T\exp\left(-C_{\mathrm{mis,2}} \sqrt{\frac{1 + \lambda^2 d_x m_{{j}}\tau_{1}}{m_j}}\right)\\
&+\Omega_4 T (1+\lambda^2)\epsilon_{\mathrm{het}}^2.
\end{align*}

\end{theorem}
\begin{proof}
We begin the proof by rewriting the regret decomposition as follows:
\begin{align*}
\mathbb{E}\big[\mathcal{R}_T^{(i)}\big]
&\le \sum_{k=2}^{k_{\mathrm{fin}}}\Bigg(
\mathbb{E}\!\Big[\mathbf{1}(\mathcal{E}_{\mathrm{est},3}^{(k-1)})\; 142\,(\tau_k-\tau_{k-1})\,\|P^{(i)}_\star\|^{8}\,
\big\|\,[\widehat A^{(i)}_{k-1}\;\widehat B^{(i)}_{k-1}]-[A^{(i)}_\star\;B^{(i)}_\star]\,\big\|_F^2\Big]  \\
&+ (\tau_k-\tau_{k-1})\,J^{(i)}(K^{(i)}_\star)
+ 4(\tau_k-\tau_{k-1}) d_u \|P^{(i)}_\star\| \sigma_k^2\,\Psi^{(i)2}_{B}
+ 2x_b^2\log T \|P^{(i)}_\star\|
\Bigg) \\
&+ T^{-1} \big( \|Q\| + 2K_b^2 \big) x_b^2 \log T + T^{-1} J^{(i)}(K^{(i)}_0) \\
&+ 24 \|P^{(i)}_{K^{(i)}_0}\| \Psi^{(i)2}_{B} (d_x + d_u) \sigma^2_w T^{-2}\log(3T)
+ 2 T^{-2} \|P^{(i)}_{K^{(i)}_0}\| \|\Theta^{(i)}_\star\|_F^2 K_b^2 x_b^2 \log T  \\
& + \sum_{k=1}^{k_{\mathrm{fin}}} 2(\tau_k - \tau_{k-1}) d_u \,\sigma_k^2
+ 3 \tau_1 \max\{ d_x , d_u \} \| P^{(i)}_{K^{(i)}_0} \| \Psi^{(i)2}_{B} - T J^{(i)}(K^{(i)}_\star),
\end{align*}
where for the adversarial setting we have
\begin{align*}
\mathbb{E}\left[\mathbf{1}(\mathcal{E}_{\mathrm{est},3}^{(k-1)})\,
\big\|\,[\widehat A^{(i)}_{k-1}\;\widehat B^{(i)}_{k-1}]-[A^{(i)}_\star\;B^{(i)}_\star]\,\big\|_F^2\right]
&\le \frac{ C_{\mathrm{stat}}\sigma^2_w (d^2_x + d_x d_u)\log(1/\delta)}{\sigma^2_{k-1} \tau_{k-1}}\left( \frac{1}{m_{{j}}} + \lambda^2 d_x\right)\\
&+ C_{\mathrm{het}}(1+\lambda)^2 \epsilon^2_{\mathrm{het}}+  C_{\mathrm{mis,1}} \exp(-C_{\mathrm{mis,2}} \sigma_{k-1}^2 \tau_{k-1}).    
\end{align*}

Applying this bound in the regret bound yields
\begin{align*}
\mathbb{E}\big[\mathcal{R}_T^{(i)}\big]
&\le \sum_{k=2}^{k_{\mathrm{fin}}}\Bigg(
142(\tau_k-\tau_{k-1})\|P^{(i)}_\star\|^{8}
\Big( \frac{ C_{\mathrm{stat}}\sigma^2_w (d^2_x + d_x d_u)\log(1/\delta)}{\sigma^2_{k-1} \tau_{k-1}}\left( \frac{1}{m_{{j}}} + \lambda^2 d_x\right) \\
&+ C_{\mathrm{het}}(1+\lambda)^2 \epsilon^2_{\mathrm{het}}+  C_{\mathrm{mis,1}} \exp(-C_{\mathrm{mis,2}} \sigma_{k-1}^2 \tau_{k-1})\Big)\Bigg) \\
&+ (\tau_k-\tau_{k-1})\,J(K^\star)
+ 4(\tau_k-\tau_{k-1})\,d_u\,\|P^{(i)}_\star\|\,\sigma_k^2\,\Psi^{(i)2}_{B}
+ 2x_b^2\log T\,\|P^{(i)}_\star\|
\Bigg) \\
&+ T^{-1} \big( \|Q\| + 2K_b^2 \big) x_b^2 \log T + T^{-1} J(K_0) \\
&+ 24 \|P^{(i)}_{K^{(i)}_0}\| \Psi^{(i)2}_{B} (d_x + d_u)\, \sigma^2_w T^{-2}\log(3T)
+ 2 T^{-2} \|P^{(i)}_{K^{(i)}_0}\| \|\Theta^{(i)}_\star\|_F^2 K_b^2 x_b^2 \log T  \\
&+ \sum_{k=1}^{k_{\mathrm{fin}}} 2(\tau_k - \tau_{k-1}) d_u \,\sigma_k^2
+ 3 \tau_1 \max\{ d_x , d_u \} \| P^{(i)}_{K^{(i)}_0} \| \Psi^{(i)}_{B^{(i)2}_\star} - T J^{(i)}(K^{(i)}_\star).
\end{align*}

Therefore by setting the exploration sequence as $\sigma_k^2 = \frac{\sqrt{d^2_u d_x}}{d_x^2 + d_xd_u}\sqrt{\frac{1 + \lambda^2 d_x m_{{j}}}{\tau_k m_{{j}}}}$, we obtain 

\begin{align*}
\mathbb{E}\big[\mathcal{R}_T^{(i)}\big]
&\le \sum_{k=2}^{k_{\mathrm{fin}}}\Bigg(
142(\tau_k-\tau_{k-1})\|P^{(i)}_\star\|^{8}
\Big( \frac{ C_{\mathrm{stat}}\sigma^2_w \sqrt{d^2_u d_x}\log(1/\delta)}{ \sqrt{\tau_{k-1}}}\sqrt{ \frac{1}{m_{{j}}} + \lambda^2 d_x} \\
&+ C_{\mathrm{het}}(1+\lambda)^2 \epsilon^2_{\mathrm{het}}+  C_{\mathrm{mis,1}} \exp\left(-C_{\mathrm{mis,2}} \sqrt{\frac{1 + \lambda^2 d_x m_{{j}}\tau_{1}}{m_j} } \right)\Big)\Bigg) \\
&+ (\tau_k-\tau_{k-1})J^{(i)}(K^{(i)}_\star)
+ 4(\tau_k-\tau_{k-1})d_u \sqrt{d^2_u d_x} \|P^{(i)}_\star\|\,\sqrt{\frac{1 + \lambda^2 d_x m_{{j}}}{\tau_{k-1} m_{{j}}}}\,\Psi^{(i)2}_{B}
\Bigg) \\
& + 2x_b^2\log T \|P^{(i)}_\star\| + T^{-1} \big( \|Q\| + 2K_b^2 \big) x_b^2 \log T + T^{-1} J^{(i)}(K^{(i)}_0) \\
&+ 24 \|P^{(i)}_{K^{(i)}_0}\| \Psi^{(i)2}_{B} (d_x + d_u)\, \sigma^2_w T^{-2}\log(3T)
+ 2 T^{-2} \|P^{(i)}_{K^{(i)}_0}\| \|\Theta^{(i)}_\star\|_F^2 K_b^2 x_b^2 \log T  \\
&+ \sum_{k=1}^{k_{\mathrm{fin}}} 2(\tau_k - \tau_{k-1}) d_u \sqrt{d^2_u d_x}\sqrt{\frac{1 + \lambda^2 d_x m_{{j}}}{\tau_{k-1} m_{{j}}}}
+ 3 \tau_1 \max\{ d_x , d_u \} \| P^{(i)}_{K^{(i)}_0} \| \Psi^{(i)2}_{B} - T J^{(i)}(K^{(i)}_\star),
\end{align*}
and by using the fact that $k_{\mathrm{fin}} \asymp \log T$ and $\sum_{k=2}^{k_{\mathrm{fin}}} \frac{\tau_k-\tau_{k-1}}{\sqrt{\tau_{k-1}}}
\leq 2\big(\sqrt{T}-\sqrt{\tau_1}\big) \leq 2\sqrt{T},$ we obtain
\begin{align*}
\mathbb{E}\big[\mathcal{R}_T^{(i)}\big]
&\leq
\Big(142\|P^{(i)}_\star\|^{8} C_{\mathrm{stat}}\sigma^2_w \log(1/\delta) + 2 d_u + 4 d_u \|P^{(i)}_\star\| \Psi^{(i)2}_{B}\Big) \sqrt{\frac{d^2_u d_x T (1 + \lambda^2 d_x m_{j})}{m_{j}}}\\
&+\Big(3 \max\{ d_x , d_u \} \| P^{(i)}_{K^{(i)}_0} \| \Psi^{(i)2}_{B} + 2x_b^2\|P^{(i)}_\star\|\Big)(\log T)^2 \\
&+ 142\|P^{(i)}_\star\|^{8} C_{\mathrm{mis,1}} \exp\left(-C_{\mathrm{mis,2}} \sqrt{\frac{1 + \lambda^2 d_x m_{{j}}\tau_{1}}{m_j}}\right) T\\
&+ 142\|P^{(i)}_\star\|^{8} C_{\mathrm{het}}(1+\lambda)^2 \epsilon_{\mathrm{het}}^2T, 
\end{align*}
which completes the proof. 
\end{proof}

\end{document}